\documentclass[10pt,journal,compsoc]{IEEEtran}

\ifCLASSOPTIONcompsoc
  \usepackage[nocompress]{cite}
\else
  \usepackage{cite}
\fi

\ifCLASSINFOpdf
\else
\fi
\usepackage{amsmath}
\usepackage{mathrsfs}
\usepackage{amsfonts,amssymb}
\usepackage{amsthm}
\theoremstyle{plain}
\newtheorem{Def}{Definition}
\newtheorem{eg}{Example}
\newtheorem{thm}{Theorem}
\newtheorem{lemma}{Lemma}
\newtheorem{cor}{Corollary}
\newtheorem{assumption}{Assumption}

\usepackage{array}
\usepackage{graphicx}
\usepackage{subfig}
\usepackage{booktabs}
\usepackage{epstopdf}

\usepackage{algorithm}
\usepackage{algorithmic}
\usepackage{setspace}

\usepackage{url}

\begin{document}
%
\title{A Generalization Theory based on Independent and Task-Identically Distributed Assumption}
%
%
%
%

\author{Guanhua~Zheng,
        Jitao~Sang,
        Houqiang~Li,
        Jian~Yu,
        and~Changsheng~Xu
\IEEEcompsocitemizethanks{\IEEEcompsocthanksitem G. Zheng is with University of Science and Technology of China (e-mail: zhenggh@mail.ustc.edu.cn).\protect\\
\IEEEcompsocthanksitem J. Sang and J. Yu are with the School of Computer and Information Technology and the Beijing Key Laboratory of Traffic Data Analysis and Mining, Beijing Jiaotong University, Beijing 100044, China (e-mail: \{jtsang, jianyu\}@bjtu.edu.cn).\protect\\
\IEEEcompsocthanksitem H. Li is with the Chinese Academy of Sciences Key Laboratory of Technology in Geo-Spatial Information Processing and Application System, Hefei 230026, China, and University of Science and Technology of China (e-mail:lihq@ustc.edu.cn).\protect\\
\IEEEcompsocthanksitem C. Xu is with the National Lab of Pattern Recognition, Institute of Automation, CAS, Beijing 100190, China, and University of Chinese Academy of Sciences(e-mail: csxu@nlpr.ia.ac.cn).
}
}

\IEEEtitleabstractindextext{%
\begin{abstract}
Existing generalization theories analyze the generalization performance mainly based on the model complexity and training process. The ignorance of the task properties, which results from the widely used IID assumption, makes these theories fail to interpret many generalization phenomena or guide practical learning tasks. In this paper, we propose a new Independent and Task-Identically Distributed (ITID) assumption, to consider the task properties into the data generating process. The derived generalization bound based on the ITID assumption identifies the significance of hypothesis invariance in guaranteeing generalization performance. Based on the new bound, we introduce a practical invariance enhancement algorithm from the perspective of modifying data distributions. Finally, we verify the algorithm and theorems in the context of image classification task on both toy and real-world datasets. The experimental results demonstrate the reasonableness of the ITID assumption and the effectiveness of new generalization theory in improving practical generalization performance.
\end{abstract}

}

\maketitle

\IEEEdisplaynontitleabstractindextext

%
\IEEEpeerreviewmaketitle

\IEEEraisesectionheading{\section{Introduction}\label{sec:introduction}}
During the past decades, machine learning generalization studies have been devoted to exploring the factors influencing model performance on unseen data, i.e., the generalization performance. Statistical learning-based solutions can be mainly divided into two groups: (1) The first group is based on the model complexities. Early studies like VC dimension \cite{vapnik2015uniform} and Rademacher Complexity \cite{bartlett2002rademacher} pointed out that the model complexities serve as the sufficient condition of PAC learnability. (2) The second group considers the influencing factors in the training process. Typical theories claimed that model generalization performance depends on the characteristics of training process like uniform stability\cite{mukherjee2006learning}, robustness\cite{xu2012robustness}, and implicit or explicit regularization\cite{bartlett1998sample}\cite{hardt2015train}.

However, in recent years, both groups fail in explaining many generalization phenomena, and thus work poorly in providing practical guidance, especially in the context of the large-scale neural network models. To name a few, the typical neural network models like VGG\cite{simonyan2014very} usually contain parameters beyond the number of the training data but mostly achieve good generalization performance. This fact contradicts the theory in the first group that higher model complexity contributes to worse generalization performance. Regarding the second group, a recent study \cite{zhang2016understanding} constructs two models with unique hypothesis space and training process but achieves a completely different generalization performance. The failure in explaining these phenomena demonstrates that there exist factors influencing the model generalization performance beyond model complexity and training process.

It is easy to conceive that the property of tasks to be solved also influences on the generalization performance. The task property often manifests itself in the data generating process and the derived data distribution. Many existing studies have actually reported such evidences: the No Free Lunch theorem~\cite{wolpert1995no} claimed that model generalization performance largely depends on specific tasks, \cite{arpit2017closer} conducted a random test experiment and found that model generalization is closely related to the characteristics of the training dataset, and \cite{Lin2017Focal} succeeded to improve model generalization by considering the dataset characteristics such as handling data imbalance problem.

A natural question arises: why most generalization theories sidestep considering these task properties? We ascribe this to the commonly imposed Independent and Identically Distributed (IID) assumption. Under the IID assumption, the tasks and the dataset characteristics are no longer necessary in quantifying the discrepancy between the training and the testing distributions, which derives the generalization bound without considering the specific data distribution. However, the IID assumption suffers from two recognized issues: (1) it does not hold in most cases due to the poor coverage of training data and the complexity of real-world tasks. \cite{swaminathan2017new} demonstrated through experiments that most of the existing datasets are non-identically distributed. \cite{shi2017beyond} found that high correlations exist in both temporal and spatial samples. Applying IID-based theories on these non-IID tasks has demonstrated devastated performance\cite{swaminathan2017new}. ~(2) The IID assumption derives a task-free generalization bound which fails to guide the practical learning tasks. One example comes from data augmentation, which demonstrates its effectiveness in improving generalization performance while adds unrealistic data and inevitably changes the distribution of the training set\cite{Zhang2018mixup} \cite{zhong2017random}. Traditional IID-based generalization theories fail to explain the improved generalization from changed data distribution or theoretically guide the new design of data augmentation solutions.

In this paper, we bypass the IID assumption and introduce a new Independent and Task-Identically Distributed (ITID) assumption. The ITID assumption well considers the properties of tasks into the data generating process. Based on the ITID assumption, we derive a new generalization bound depending on both the task property and model invariance capability. To leverage the new generalization bound for practical guidance, we further provide operational algorithms to help improve model generalization performance by optimizing the training data distribution and enhancing the model invariance capability. The main contributions of this work can be summarized in four-fold:

\begin{itemize}
\item Data generating assumption: we introduce a new ITID assumption to the data generating process, which enjoys two advantages: (1) it is more realistic to fit most real-world tasks; (2) it well considers the task property and provides prescriptive guidance for generalization improvement.
\item Theoretical generalization bound: we derive task-related generalization bound based on the ITID assumption in both special and general cases. The new generalization bound well explains many phenomena conflicting with the existing generalization theories.
\item Algorithmic prescriptive guidance: we design a data augmentation algorithm with a theoretical guarantee to improve the generalization capability
by enhancing model invariance.
The algorithm serves as one inspiring example to leverage the proposed ITID-based theory to improve generalization performance in practical tasks.
\item Extensive experimental validation: we conduct experiments on both the toy data and the real-world Cifar-10 dataset to validate the derived task-based generalization bound as well as the effectiveness of the invariance enhancement algorithm in improving generalization performance.
\end{itemize}

In the remainder, in Section \ref{s2}, we provide the necessary notations and introduce related work to better position our work. In Section \ref{s3}, we formally introduce the ITID assumption and derive the task-related generalization bound under the ITID assumption. In Section \ref{s4}, we apply the new generalization theory to propose a practical data augmentation algorithm to improve generalization performance. In Section \ref{exp}, we report the experimental results on toy data and Cifar-10 dataset to validate the derived theories. Section \ref{s7} concludes this paper with discussions and future work.

\section{Preliminary and Related Work}
\label{s2}
PAC-learnability serves as the basis of most generalization theories. To better understand as well as position our work, this section first introduces some key notations and then reviews the related work in PAC-learnability with its many variants.

\subsection{Notations}
To guarantee a consistent understanding, we follow the typical notations used in traditional machine learning theories \cite{Shalev2014Understanding}. We denote $z$ as a data instance which has two parts: the input $x \in \mathcal{X}$ and the label $y\in \mathcal{Y}$, where $\mathcal{X}$ is the input space and $\mathcal{Y}$ is the label space.
Denoting $\mathcal{H}$ as the hypothesis set, we use $h\in\mathcal{H}:\mathcal{X}\mapsto \mathcal{Y}$ to indicate the specific hypothesis function that maps the input to the label space. With training set $S=\{z_1,z_2,...,z_n\}$ as a set of $n$ instances, the empirical risk for the hypothesis $h$ on the training set $S$ is denoted as $L_{S}(h):\mathcal{H} \mapsto \mathbb{R}_{+}$ and calculated as:
\begin{equation}
L_{S}(h)=1-\frac{1}{n}\sum_{i=1}^{n}\mathbb{I}(h(x_{i}),y_{i})
\end{equation}
where $\mathbb{I}$ is an indicator function outputing 1 if $h(x_{i})=y_{i}$, and 0 if otherwise.

\subsection{PAC Learnability Theory}
Generalization theories concern more on the expected prediction risk in practice instead of the empirical risk in the training set. PAC learning theory\cite{valiant1984theory} serves to model the expected risk based on the generating process of the dataset. Using $D_\mathcal{X}$ to denote a distribution over the input space $\mathcal{X}$, generating the training set $S$ involves first sampling the input $x$ from $D_\mathcal{X}$, and then obtaining the label $y$ from an unknown target function $f:\mathcal{X}\mapsto \mathcal{Y}$. Under this data generating process, the expected risk $L_{D_\mathcal{X},f}:\mathcal{H} \mapsto \mathbb{R}_{+}$ for given hypothesis $h$ is calculated as:
\begin{equation}
L_{D_\mathcal{X},f}(h)=\mathbb{P}_{x \sim D_\mathcal{X}}[h(x)\neq f(x)]=D_\mathcal{X}(\{x:h(x)\neq f(x)\})
\end{equation}

Based on the above expected risk, \emph{PAC learnability} is formally defined as follows \cite{valiant1984theory}:

\emph{
A hypothesis class $\mathcal{H}$ is PAC learnable if there exists a function $m_{\mathcal{H}}: (0,1)^{2} \mapsto \mathbb{N}$ and a learning algorithm with the following property:
$\forall \epsilon,\delta \in(0,1)$, if the number of training set $n \geq m_{\mathcal{H}}(\epsilon,\delta)$, the algorithm returns a hypothesis $h\in\mathcal{H} $ such that,
with probability of at least $1-\delta$,
}
\begin{equation}
L_{D_\mathcal{X},f}(h)\leq \epsilon
\end{equation}

PAC learnability enables the evaluation of whether a correct hypothesis can be learned from the training set and also derives an upper bound to the expected risk when it is PAC learnable. Notably, there are two assumptions behind the PAC learning theory: (1) realizability assumption, that there exists a perfect hypothesis $h^{*}\in \mathcal{H}$ satisfying $L_{D_\mathcal{X},f}(h^{*})=0$; and (2) IID assumption, that the data is independently drawn from a unique distribution. Since these assumptions are not naturally held in practice, many previous studies have been devoted to relevant modifications and discussions, which are reviewed in the next two subsections respectively.

\subsection{Agnostic PAC Learnable Theory}
The realizability assumption is hard to satisfy since the target function $f$ is agnostic in most cases. Agnostic PAC learning\cite{haussler2018decision} is proposed to address this problem by assuming a new data generating process that an instance $z=(x,y)$ is directly drawn from the distribution $D_{\mathcal{X,Y}}$ over the instance space $\mathcal{X}\times \mathcal{Y}$. In this context, the target function $f$ is no longer needed, and the expected risk can then be calculated as:
\begin{equation}
\begin{split}
L_{D_{\mathcal{X,Y}}}(h)&=\mathbb{P}_{(x,y) \sim D_{\mathcal{X,Y}}}[h(x)\neq y]\\
&= D_{\mathcal{X,Y}}(\{(x,y): h(x)\neq y\})
\end{split}
\end{equation}

\noindent Agnostic PAC learnability is formally defined as follows:

\emph{
A hypothesis class $\mathcal{H}$ is agnostic PAC learnable if there
exist a function $m_{\mathcal{H}}: (0,1)^{2} \mapsto \mathbb{N}$ and a learning algorithm with the following property:
$\forall \epsilon,\delta \in(0,1)$, if the number of training set $n \geq m_{\mathcal{H}}(\epsilon,\delta)$, the algorithm returns a hypothesis $h\in\mathcal{H} $ such that,
with probability of at least $1-\delta$,
}
\begin{equation}\label{tl}
L_{D_{\mathcal{X,Y}}}(h)\leq \min\limits_{h'\in \mathcal{H}} L_{D_{\mathcal{X,Y}}}(h')+\epsilon
\end{equation}

Agnostic PAC learnability relaxes the realizability assumption and theoretically bounds the expected risk discrepancy between the given hypothesis and the best hypothesis in the same class. However, since the real distribution $D_{\mathcal{X,Y}}$ is still unknown in most cases, modeling the discrepancy to the best hypothesis is not adequate to evaluate the actual generalization performance of the given hypothesis. Later generalization theories use its sufficient condition--\emph{uniform convergence}, to derive bound to the discrepancy between empirical and expected risk (i.e., generalization gap) so as to evaluate the practical generalization performance.

\noindent The uniform convergence states that:

\emph{
A hypothesis class $\mathcal{H}$ is agnostic PAC learnable if there
exist a function $m_{\mathcal{H}}: (0,1)^{2} \mapsto \mathbb{N}$ with the following property:
for $\forall \epsilon,\delta \in(0,1)$, for any hypothesis $h \in \mathcal{H}$, with probability of at least $1-\delta$,
 \begin{equation}\label{uni}
\left | L_{D}(h)-L_{S}(h) \right | \leq \epsilon
\end{equation}
}

Following uniform convergence, many previous generalization theories employ the IID assumption to instantiate the bound $\epsilon$ in the above equation, with terms of model complexity like VC dimension \cite{vapnik2015uniform} and Rademacher Complexities \cite{bartlett2002rademacher}. Our new generalization theory in this paper is also based on the classical uniform convergence and concerns the discrepancy between the empirical risk and the expected risk. However, without employing the IID assumption, we assume a new task-correlated data generating process and derive the bound of generalization gap involving with also the task property and data distribution.

\subsection{Non-IID Generalization Theories}
Non-IID generalization theories origin from the needs to address tasks that either violate the independently distributed assumption like recommender systems involving with time series samples, or violate the identical distributed assumption like active learning with distribution and target changing over time.

To solve these non-IID tasks, pioneering researchers have managed to relax the IID assumption mainly in two directions. Some researchers make efforts to model the dependence in sampling, and assume that the data are drawn from $\alpha / \beta / \phi$-mixing process \cite{yu1994rates}\cite{mohri2008stability}\cite{mohri2009rademacher}. Meanwhile, some other studies attempt to address the changes in distributions, by imposing additional assumptions for the changes \cite{bartlett1992learning} \cite{helmbold1994tracking} \cite{bartlett2000learning} \cite{mohri2012new}. There also exist studies combining the two directions, like \cite{ralaivola2010chromatic}.

Instead of solving specific tasks which obviously violating the IID assumption, this paper is motivated to address more general learning tasks which are conventionally treated as satisfying the IID assumption in generalization analysis. Taking the task of image classification as an example, we relax the identically distributed assumption and assume that image data is generated related to its belonging category. Since the focus of this paper is actually on modifying the identical distribution assumption, in the following, we review in more detail the related studies in the direction of independent but non-identically distributed generalization theories and discuss their relation to this paper.

The main idea of the past independent non-identically distributed  generalization theories is adding a new constraint, which is weaker than the IID assumption, to the data generating process. For example, \cite{bartlett1992learning}\cite{barve1997complexity}\cite{long1999complexity} constrain the L1-Norm between distributions, \cite{helmbold1994tracking} needs the number of distributions to be fixed, \cite{freund1997learning} assumes that the distribution changes follow simple structures, \cite{bartlett2000learning} considers a condition that distribution changes are arbitrary but happen rarely, \cite{mohri2012new} introduces a new measure of discrepancy from domain adaption to substitute the L1-Norm.

These work suffer from two main problems. Firstly, these studies are motivated to address specific tasks and the imposed constraints only work on the corresponding data generating process. This limits its application to general learning tasks. Secondly, these studies focus more on quantifying the discrepancy between the real and IID assumption, with derived descriptive conclusions that larger discrepancy leads to weaker generalization bounds. However, the distribution discrepancy influence generalization in a more complex way. For instance, the data augmentation methods like mix-up\cite{Zhang2018mixup} and random-erasing\cite{zhong2017random} generate instances which do not actually exist. Since the added unrealistic data inevitably changes the data distribution, those generalization theories simply constraining the distribution discrepancy fails to explain the observed improved generalization performance.

Our new generalization theory provides a solution to both problems: (1) The new data generating assumption is addressing more general tasks and expected to be satisfied in most cases. We evaluate the generalization performance based on model capability instead of on the deviated data distribution, making it fitting to different data generating scenarios. (2) By introducing task-correlated generative variables into the data generating process, we succeed to examine the correlation between generalization performance and task properties. The derived generalization bound is related to both the task complexity and model invariance capability. Furthermore, based on the task-related generalization bound, we provide prescriptive guidance to improve generalization performance in practical tasks.

\begin{figure*}[!t]
\centering
\includegraphics[width=5in]{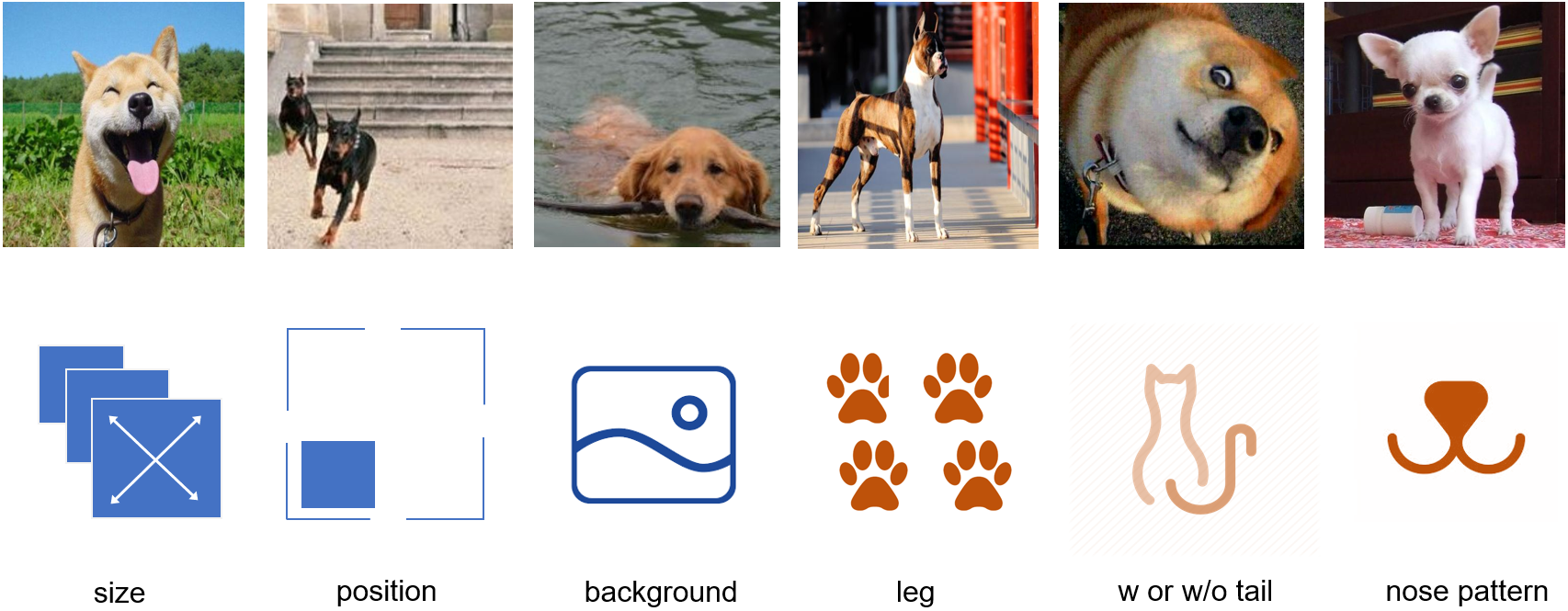}
\caption{Examples of (Top row): images belonging to the same category; and (Bottom row): generative variables controlling the image generating process.}
\label{fig1}
\end{figure*}

\begin{figure*}[!t]
\centering
\subfloat[]{\includegraphics[width=1.6in]{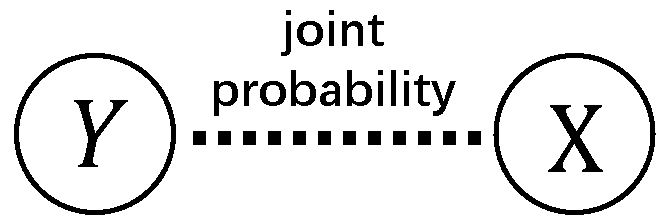}\label{fig2-a}}
\hfil
\subfloat[]{\includegraphics[width=2.2in]{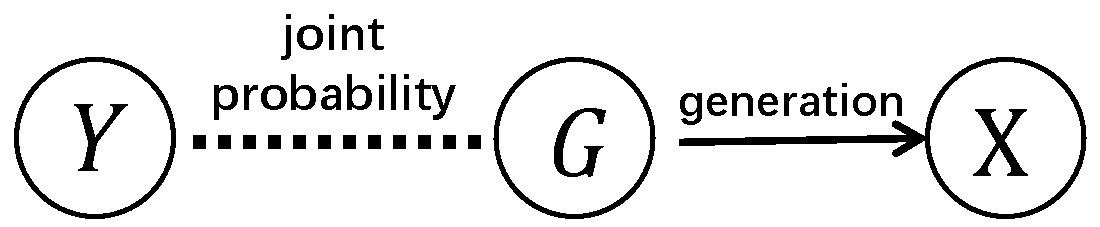}\label{fig2-b}}
\hfil
\subfloat[]{\includegraphics[width=2.2in]{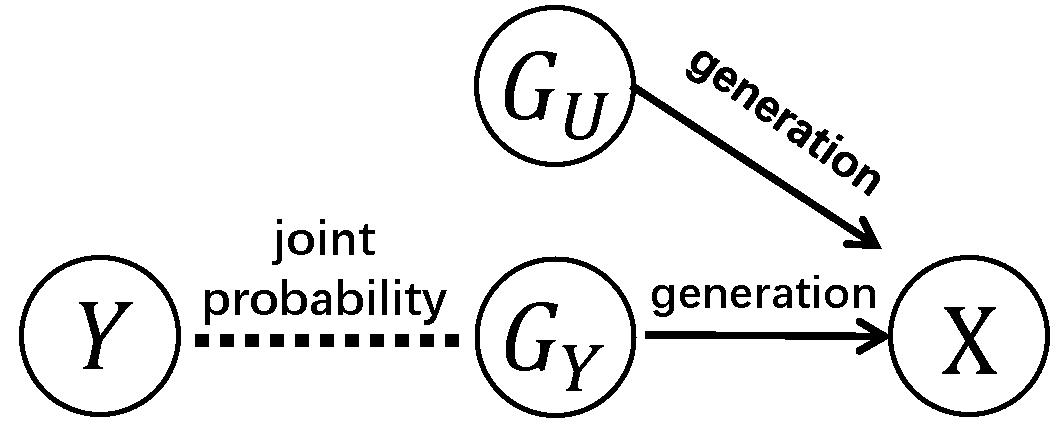}\label{fig2-c}}

\caption{Data generating process on: (a) original instance space  (b) generative variables; (c) task-correlated and -uncorrelated generative variables. }
\label{fig2}
\end{figure*}
\section{ITID Assumption and Task-related Generalization Bound }
\label{s3}
This section first introduces the ITID assumption to consider task property into the data generating process. Based on the ITID assumption, task-related generalization bound is then derived in both special and general cases.
\subsection{Generative Variable and ITID Assumption}
Before elaborating the ITID assumption, we first introduce an alternative way to examine the data generating process. We discuss that one reason why the conventional IID assumption is violated in most scenarios is that the assumption directly applies on the original instance space $\mathcal{X}\times \mathcal{Y}$. Taking an image classification task as example, the input images even belonging to the same category may vary significantly from each other (Fig.1-top illustrates some dog images with very different appearances). It is easy to see that the data patterns in the input instance space are too complex to constrain.

To address the complexity in instance space, we introduce an intermediate level to the data generating process which is controlled by \emph{Generative Variable}. We denote the generative variables as those unique factors leading to various data input patterns. As illustrated in Fig.1-bottom, the generative variables controlling the generating of the dog images include the dog size, position, background and legs, tail, nose patter, etc. These generative variables reduce the complexity in generating data inputs and thus facilitate the analysis of data distributions and generalization performance in a controllable level.

We now formalize the data generating process based on the above introduced generative variables. Suppose the target task involves with $m$ generative variables: $G=\{G^{(i)}\}_{i=1,2,...,m}$.  We further use $\textbf{g}=(g^{(1)},g^{(2)},...,g^{(m)})$ to denote a specific instantiation of random vectors $\textbf{G}=(G^{(1)},G^{(2)},...,G^{(m)})$, and $\mathcal{G}$ is a set which consists of all the possible values of $\textbf{G}$. The process of input generating from generative variables is encoded in the generating function $\phi:\mathcal{G} \mapsto \mathcal{X}$. $x_i=\phi(\textbf{g}_i)$ denotes the generation of input $x_i$ from certain generative variable values $g_i$. The data generating process based on the generative variables is defined as follows:
\begin{Def}\label{def4}
(Generative Variable-based Data Generating Process).
To generate each instance $z_i=(x_i,y_i)$, (1) sampling the generative variable and the label $d_i=(g_i,y_i)$ from exemplar distribution $D_i$ over $ \mathcal{G} \times \mathcal{Y}$\footnote{We will use exemplar to denote the sample from $\mathcal{G}\times \mathcal{Y}$ and instance to denote the sample from $\mathcal{X}\times \mathcal{Y}$ in the rest of this paper.}; (2) generating input from the generating function $x_i=\phi(\textbf{g}_i)$.
\end{Def}

The data generating process on the original instance space and generative variables is illustrated in Fig.2(a),(b) respectively.

We now consider the task properties into the generative variable-based data generating process. By examining the generative variables illustrated in Fig.1-bottom, we find that although all the generative variables contribute to the generation of input $X$, only part of them are closely correlated to the tasks. Coarsely representing the task as the target concept label $dog$, the left 3 generative variables in Fig.1-bottom may lead to different generations of $X$, but only have unstable correlations with the label $dog$. We call these generative variables as \emph{task-uncorrelated generative variables} $G_U$ if they do not necessarily change as the task (e.g., the target concept) changes. The right 3 generative variables in Fig.1-bottom, however, are bound to be different when the target concept $dog$ appears or not. We call these generative variables as \emph{task-correlated generative variables} $G_Y$ if they necessarily change as the task changes. The complete data generating process on task-correlated and task-uncorrelated
generative variables are illustrated in Fig. 2(c).

After partitioning the generative variables according to their correlation to tasks, we re-examine the problem of IID assumption. If we assume the generating function from $\mathcal{G}$ to $\mathcal{X}$ is a one-to-one mapping, the IID assumption over $\mathcal{X}\times\mathcal{Y}$ equivalently applies over $\mathcal{G}\times\mathcal{Y}$, which constrains that for all generated instances the sampled task-uncorrelated generative variables have the identical distribution. With \emph{background} as the example task-uncorrelated generative variable and dogs in the training image set are with a grass background, IID assumption actually imposes an unrealistic assumption that for all test cases the dogs also have a grass background.

The above discussion inspires our new assumption on the data generating process. Instead of constraining the original instance space or all the generative variables, we assume identical and independent distribution only on the task-correlated generative variables. This leads to the new ITID assumption which is formally defined as follows:

%

\begin{assumption}\label{def5}
(Independent and Task-Identically Distributed Assumption) The exemplars $\{d_1,d_2,...,d_n\}$ in the dataset $S$ are Independent and Task-Identically Distributed according to the distribution $D$. That is, every $d_i\in \mathcal{G}\times\mathcal{Y}$ in $S$ is sampled from a corresponding distribution $D_i$ independently~\footnote{For simplicity, we assume the generative variables are discrete.}, and there exist a subset of generative variables $G_Y\subseteq G$ that all $D_i$ share the identical marginal distribution $D$ over $\mathcal{G_Y}\times\mathcal{Y}$:
\begin{equation}
\sum_{G_Y^c}D_i(\boldsymbol{G},Y)=D(\boldsymbol{G_Y},Y), \forall i=1,2,...,n
\end{equation}
where $G_Y^c$ is the complementary set of $G_Y$ in $G$.
\end{assumption}

In ITID assumption, we call $G_Y$ task-correlated generative variables, and $G_U=G_Y^c$ task-uncorrelated generative variables.

 We can see that the ITID assumption is a reasonable relaxation to the IID assumption. By constraining the identical assumption only on the task-correlated generative variables, ITID allows the data to be sampled from different instance distributions. The differentiation of generative variables according to their correlation to the solved tasks enjoys two advantages: (1) The introduction of task-uncorrelated generative variables succeeds to maintain the diversity in the input space $\mathcal{X}$~\footnote{It is recognized in previous studies that task-uncorrelated inputs can influence the generalization performance\cite{littlestone1988learning}. However, the correlations to tasks were examined at the original instance level (e.g., image pixels for image classification task), which limits its applicability.}, e.g., the variables of \emph{brightness, color, background, etc.} can address the major variation in images that do not decisively affect the target concept label. (2) Identical distribution is necessary to evaluate the generalization performance on unseen data based on observed data. The introduction of task-correlated generative variables
  preserves this necessity on conditions of not violating the real data distribution as much as possible. It also enables the consideration of task properties into the data generating process and later generalization analysis.

\subsection{Generalization Bound based on the ITID Assumption}
Compared with the IID assumption, the ITID assumption describes the relationships between the generative variables and the tasks, and relaxes the constraints on task-uncorrelated generative variables. Under the ITID assumption, we can theoretically analyze how the task properties influence the generalization performance.

Since the ITID assumption is based on the generative variables, to analyze the generalization performance under the ITID assumption, we further introduce another definition regarding the hypothesis characteristic on generative variables, \emph{hypothesis determination}.

\begin{Def}\label{def6}
(Hypothesis determination on generative variables) Given a subset of generative variables $G_l \in G$, hypothesis $h\in \mathcal{H}$ is $G_l$-determining if and only if~\footnote{It is noted that the hypothesis determination can be equivalently defined by mutual information: $I(\hat{Y};\boldsymbol{G_l^c}|\boldsymbol{G_l})=0$, (where $G_l^c$ is the complementary set of $G_l$ in $G$: $G=\{G_l,G_l^c\})$: the predictions $\hat{Y}$ is determined by the generative variables $G$ indicating that $H(\hat{Y}|\boldsymbol{G_l},\boldsymbol{G_l^c})=0$, therefore $I(\hat{Y};\boldsymbol{G_l^c}|\boldsymbol{G_l})=H(\hat{Y}|\boldsymbol{G_l})-H(\hat{Y}|\boldsymbol{G_l},\boldsymbol{G_l^c})=H(\hat{Y}|\boldsymbol{G_l})$}:

\begin{equation}
\label{e9}
H(\hat{Y}|\boldsymbol{G_l})=0
\end{equation}
where $\hat{Y}$ is the prediction of hypothesis $h$, and $H(\hat{Y}|\boldsymbol{G_l})$ is the condition entropy. In addition, hypothesis space $\mathcal{H}$ is $G_l$-determining, if $\forall h\in\mathcal{H}$ is $G_l$-determining.
\end{Def}

According to the definition, $H(\hat{Y}|\boldsymbol{G_l})$ suggests that the prediction $\hat{Y}$ is uniquely and sufficiently determined given the generative variables $\boldsymbol{G_l}$. In other words, for any two input $x_i$, $x_j$, if their generative variables $G_l(x_i)=G_l(x_j)$, their hypothesis predictions are guaranteed to be the same:
$h(x_i)=h(x_j)$ with probability 1. The determination on $G_l$ also indicates that the hypothesis prediction is invariant to other generative variables, $G_l^c$. Therefore,
if a hypothesis $h$ is $G_l$-determining, we can alternatively call $h$ is $G_l^c$-invariance. In the following, we provide an example to help understand this definition and also instantiate generative variables in specific tasks.

\begin{eg}\label{eg1}
Suppose we have a specific task with input of three dimensions $\boldsymbol{X}=(X^{(0)},X^{(1)},X^{(2)})$. The generating function is identity mapping $\phi(g)=g$, and thus the generative variable set $\boldsymbol{G}=\boldsymbol{X}=(X^{(0)},X^{(1)},X^{(2)})$.
Given the hypothesis space is $\mathcal{H}=\{h|\hat{Y}=h(\boldsymbol{X})=a_0X^{(0)}+a_1X^{(1)}, a_0,a_1\in \mathbb{R}\}$, then the predictions of the hypotheses in $\mathcal{H}$ only depend on $\{X^{(0)},X^{(1)}\}$ and $H$ is thus $\{X^{(0)},X^{(1)}\}$-determining or $\{X^{(2)}\}$-invariance.
\end{eg}

According to the definition of hypothesis determination, it is desired that the hypothesis predictions only depend on the task-correlated generative variables. In the above example, if task-correlated generative variable set $G_Y=\{X^{(0)},X^{(1)}\}$, we call $\mathcal{H}$ is $G_Y$-determining or $G_U$-invariance. The predictions of the $G_Y$-determining models will not be influenced by the task-uncorrelated generative variables. As the marginal distributions over the task-correlated generative variables are identical according to the ITID assumption, the generalization bounds and the statistical inequalities which are based on the IID assumption still hold. In the following, we first derive the generalization bound for this special case when the hypothesis is $G_Y$-determining:
\begin{thm}\label{thm2}
Suppose that hypothesis space $\mathcal{H}$ is $G_{Y}$-determining, $\forall \epsilon,\delta \in(0,1)$, for any hypothesis $h \in \mathcal{H}$, with probability of at least $1-\delta$, we have
\begin{equation}
\left | L_{S}(h)-L_{D}(h) \right |\leq \sqrt{\frac{2(TKln2+ln(1/\delta))}{n}}
\end{equation}
where $n$ is the size of the training set $S$, $K$ is the number of labels in the label space $\mathcal{Y}$, and $T$ is the number of different variable configurations of $\boldsymbol{G_Y}$.
\end{thm}

The proof of Theorem \ref{thm2} is provided in Appendix-A.
This generalization bound is different from the traditional bounds as it shows that the generalization performance also depends on the task complexity $T$ and the number of label classes $K$. This provides an alternative way to examine generalization performance beyond the model complexity and training process. However, $G_Y$-determining is an ideal assumption, which is difficult to be satisfied in most cases. We are more interested to derive the new task-related generalization bound in general cases, so as to improve generalization performance in practical tasks. To this goal, we introduce the following definition to record the level of violating the $G_Y$-determining (i.e., dependence on $G_U$) for a specific hypothesis.

\begin{Def}\label{def7}
($\gamma$-$G_{U}$-dependence)
Hypothesis $h$ is $\gamma$-$G_{U}$-dependence($\gamma\geq0$), if~\footnote{Similar to $G_l$-determining, $\gamma$-$G_U$-dependence can also be defined using mutual information as:$I(\hat{Y};\boldsymbol{G_U}|\boldsymbol{G_Y})\leq \gamma$ }:
\begin{equation}
H(\hat{Y}|\boldsymbol{G_Y})\leq \gamma
\end{equation}
Hypothesis space $\mathcal{H}$ is $\gamma$-$G_{U}$-dependence, if $\forall h\in\mathcal{H}$ is $\gamma$-$G_{U}$-dependence.
\end{Def}

$\gamma$-$G_U$-dependence makes it possible to analyze more general cases when $G_Y$-determining is not satisfied. By varying the value of $\gamma$,  $\gamma$-$G_U$-dependence can theoretically cover arbitrary learning tasks. Therefore, we derive the following task-related generalization bound for general cases:

\begin{thm}\label{thm3}
Suppose hypothesis space $H$ is $\gamma$-$G_{U}$-dependence: if $n\geq m_{\mathcal{H}}(\frac{\epsilon}{2},\delta)$, with probability of at least $1-\delta$, $\forall h \in \mathcal{H}$ we have
\begin{equation}\label{12}
L_{D}(h)\leq \min\limits_{h'\in \mathcal{H}} L_{D}(h')+\epsilon+\frac{\gamma}{log2}
\end{equation}
where $m_{\mathcal{H}}$ is derived from any uniform convergency bound under $G_Y$-determining constraint: if $n\geq m_{\mathcal{H}}(\epsilon,\delta)$, with probability of at least $1-\delta$, we have
\begin{equation}
\left | L_{S}(h)-L_{D}(h) \right |\leq \epsilon
\end{equation}
\end{thm}
The proof of Theorem \ref{thm3} is provided in Appendix-B.
The three terms in the above generalization bound in Eqn.~\eqref{12} indicate respectively the expected risk for the best hypothesis, the generalization upper bound for $G_Y$-determining (which is related to the task complexity as stated in Theorem 1), and the influence of task-uncorrelated generative variables on hypothesis prediction. Give specific task, the generative variables and the task complexity are fixed, therefore the practical generalization performance is largely affected by the third term $\gamma$, i.e., how much the prediction is influenced by the task-uncorrelated generative variables.

 We are now ready to use the new generalization bound to explain the random label phenomenon\cite{zhang2016understanding} mentioned in the Introduction. After changing the instance labels, the marginal exemplar distribution $D_i(\boldsymbol{G_Y}, Y)$ of the training set is no longer identical to the testing set. It can also be viewed that the original task-correlated generative variables transfer to task-uncorrelated generative variables. The new task-uncorrelated generative variables increase the influence $\gamma$ to prediction, which leads to the decreased generalization performance. We can also understand this theorem by recalling the continuous efforts in obtaining invariance in computer vision studies. Viewing position, size, angle as the task-uncorrelated generative variables, the shared goal is to obtain the translation, scale and rotation invariance in both hand-crafted and deep neural features\cite{Bengio2013Representation}. We can see that the new generalization bound provides a theoretical explanation to the improved generalization performance from feature invariance.

\section{Application: Enhancing the Invariance over Task-uncorrelated Generative Variables}
\label{s4}
 According to the above ITID-based generalization theory, models will have lower generalization bound by increasing its invariance over task-uncorrelated generative variables. This section aims to apply the proposed generalization theory to provide prescriptive guidance to improve generalization performance. Specifically, we derive three additional theorems to respectively address the feasibility of manipulating single generative variables, the measure to evaluate the influence of generative variables, and the way to obtain invariance via balancing training data distribution. We then employ these theorems to explain the effectiveness of existing data augmentation solutions.

\subsection{Addition Rule for Model Influence  }
In Theorem \ref{thm3}, $\gamma$ represents the upper bound of $I(\hat{Y};\boldsymbol{G_U}|\boldsymbol{G_Y})$, which is the influence of all the task-uncorrelated generative variables on model predictions. However, as the task-uncorrelated generative variables $G_U$ are hard to enumerate, it is difficult to decrease $\gamma$ directly on the whole set. Therefore, we propose the following theorem that there exist an upper bound of model influence follows addition rule, which ensures the feasibility of reducing model influence by manipulating single generative variable.

\begin{thm}\label{thm4}
Suppose that $\hat{Y}$ is the predictions of hypothesis $h$, and we have $N$ task-uncorrelated generative variables: $G_U=\{G_U^{(i)}\}_{i=1..N}$. If
\begin{equation}
\gamma=\sum_{i=1}^{N}{I(\hat{Y};G_U^{(i)}|\boldsymbol{G_Y})}
\end{equation}
then $h$ is $\gamma$-$G_{U}$-dependence. Where $I$ represents the mutual information.
\end{thm}

The proof of Theorem \ref{thm4} can be seen in Appendix-C. It divides the total influence of all the task-uncorrelated generative variables into the sum of the influences of each task-uncorrelated generative variable. Therefore, if we have some prior knowledge for a part of the task-uncorrelated generative variables, we can decrease $\gamma$ by attenuating the influence of these known task-uncorrelated generative variables.

\subsection{Choosing Influential Generative Variables}
 As the model capacities are limited, people can not choose to obtain the invariance of all the variables. The problem thus turns to choose the most influential task-uncorrelated generative variables.

Theorem \ref{thm4} provides off-the-shelf measure $I(\hat{Y};G_U^{(i)}|\boldsymbol{G_Y})$  to evaluate the model influence of each task-uncorrelated generative variable $G_U^{(i)}$. However, its value is hard to calculate: (1) the task-correlated generative variables $G_Y$ is difficult to enumerate; (2) $\hat{Y}$ represents the predictions of the final models, which is unknown before training when evaluating the influence of $G_U^{(i)}$.

We address the first problem by using $I(\hat{Y};G_U^{(i)})$ to estimate the value of $I(\hat{Y};G_U^{(i)}|\boldsymbol{G_Y})$. Note that we are not interested in the absolute value of this conditional mutual information for each $G_U^{(i)}$. Calculating the relative value is sufficient to determine which task-uncorrelated generative variable is more influential.  Since $G_Y$ is identical for all the task-uncorrelated generative variables, removing it will not affect the evaluation of relative influence.

For solving the second problem, the key is how to effectively predict the properties of the final model before training it. We consider exploring the relationship between the empirical distribution and the optimal predictions.
To quantify this relationship, we first expand the output of the hypothesis to real numbers \cite{haussler2018decision}: from $\mathcal{Y}$ to $[0,1]^K$, where $K$ is the number of the labels. Then we introduce a new definition of the hypothesis $h:\mathcal{X}\mapsto [0,1]^K$. For clarity, we denote $\boldsymbol{Q}=(Q_0(x),Q_1(x),...,Q_K(x))=h(x)$ to represent the output scores when $x$ is the input. In general, if $x$ is given, the sum of the scores of all possible labels will be 1: $\sum_{y\in\mathcal{Y}}{Q_y(x)}=1$.
After modifying the definition of the hypothesis, the definition of $G_l$-determining also changes accordingly as follows.

\begin{Def}\label{def9}
(strict-$G_l$-determining or strict-$G_l^{c}$-invariance)
Suppose that $G$ is the whole generative variables of the task,
$G_l$ is a subset of $G$, and $G_l^{c}$ is the complementary set of $G_l$ in $G$.
The hypothesis $h$ is called strict-$G_l$-determining or strict-$G_l^{c}$-invariance, means that
\begin{equation}
I(\boldsymbol{Q};\boldsymbol{G_l^c}|\boldsymbol{G_l})=0
\end{equation}
where $I$ is the mutual information.

As $\boldsymbol{Q}$ is a function of $\boldsymbol{G}=(\boldsymbol{G_l},\boldsymbol{G_l^c})$, $I(\boldsymbol{Q};\boldsymbol{G_l^c}|\boldsymbol{G_l})=0$ means that $\boldsymbol{Q}$ can be represented by a function of $\boldsymbol{G_l}$ with probability 1:
\begin{equation}
\boldsymbol{Q}=h(x)=\psi_h(\boldsymbol{G_l}=G_l(x))
\end{equation}
\end{Def}

Under the constraint of strict-$G_l$-determining, the output values of the hypothesis can be calculated as a function over $\boldsymbol{G_l}$. In this way, we can use the empirical distribution of the training set over $(\boldsymbol{G_l},Y)$ to calculate the cross-entropy loss. The following Theorem 4 is derived to show that if the optimal models obtain minimum cross-entropy loss, their output values will be determined by the empirical distribution over $(\boldsymbol{G_l},Y)$.

\begin{thm}\label{thm5}
Suppose that the hypothesis class $\mathcal{H}$ is strict-$G_l$-determining, and the optimal hypothesis $h'\in\mathcal{H}$ is obtained by cross-entropy loss with empirical risk minimization. The output values $(Q_0^{'},Q_1^{'},...,Q_K^{'})$ of the optimal hypothesis $h'(x)$ under such settings will be
\begin{equation}
Q_y^{'}(x)=\widetilde{P}_S(Y=y|\boldsymbol{G_l}=G_l(x))
\end{equation}
where $\widetilde{P}_S$ represents the empirical distribution of training set $S$.
\end{thm}

The proof of Theorem \ref{thm5} can be seen in Appendix-D.
According to Theorem \ref{thm5}, the output values of the optimal models can be estimated by the empirical distributions of training set: we can use $I_S(Y;G_U^{(i)})$ to estimate $I(\hat{Y};G_U^{(i)})$, where $I_S$ means the empirical mutual information. Additionally, $I_S(Y;G_U^{(i)})$ can be further calculated as
\begin{equation}
I_S(Y;G_U^{(i)})=H_S(Y)-H_S(Y|G_U^{(i)})
\end{equation}
where $H_S$ means the empirical entropy. As $H_S(Y)$ is fixed in the training set, we finally transfer the evaluation of model influence to calculating the conditional entropy $H_S(Y|G_U^{(i)})$ from the empirical training distribution: the smaller the $H_S(Y|G_U^{(i)})$, the higher the influence that $G_U^{(i)}$ has on the model predictions. That is to say, a task-uncorrelated generative variable with heavy distribution bias on labels $Y$, will have a large impact on the model generalization performance. The above theorem is also consistent with the observations in agnostic data selection bias and causality feature learning\cite{shen2018causally}. Still using the task-uncorrelated background in dog image classification task as example, if the grass background appear with a dog in all training images, we will have $H_S(Y|G_U^{(i)}=grass)=0$. According to Theorem 4, it indicates that the grass background will seriously influence the predictions of the final models, and may deteriorate the model generalization if background grass in the testing images appear without dogs.

\subsection{Obtaining Invariance by Balancing Training Set}
After choosing the influential task-uncorrelated generative variables, the remaining problem is how to obtain invariance over these generative variables. There are two major ways to achieve this goal. One is designing models obtaining more invariance on task-uncorrelated generative variables. For example, CNNs exploit the convolutional and pooling layers to obtain translation invariance to some extent \cite{Y1997Handwritten}\cite{Bengio2013Representation}. The other is modifying the data distributions, by explicitly importing new data via data augmentations\cite{krizhevsky2012imagenet} or implicitly constraining the loss functions\cite{bartlett1998sample}.  This subsection fails into the second way and introduces the following theorem to reduce the influence of task-uncorrelated generative variables by explicitly balancing training data distributions. We will leave the attempts of model design and other distribution modification solutions in future studies.
\begin{thm}\label{thm6}
Suppose that the hypothesis class $\mathcal{H}$ is strict-$G_l$-determining, and obtain the optimal hypothesis $h'\in\mathcal{H}$ is obtained by cross-entropy loss with empirical risk minimization. Given a generative variable subset $G_t\subset G_l$, we have the sufficient condition for that the optimal $h'$ is strict-$G_t$-invariance on the training set $S$:

 $G_t$ is statistically independent with $G_t^c$ and $Y$ on the training set $S$.

\noindent where $G_t^c$ is the complementary set of $G_t$ in $G_l$.
\end{thm}

The proof of Theorem \ref{thm6} can be seen in Appendix-E.
Note that researchers usually use the label with the largest output value as the prediction of the model: $\arg\max\limits_{y\in\mathcal{Y}}{Q_y}$, we provide the following corollary for such predictions:
\begin{cor}\label{cor1}
Suppose that $Q_y$ is the output value of hypothesis $h \in \mathcal{H}$ for label $Y=y$, the sufficient condition of that $\arg\max\limits_{y\in\mathcal{Y}}{Q_y}$ is $G_t$-invariance is:

$h$ is strict-$G_t$-invariance
\end{cor}
Therefore, assume $G_t$ contains many influential task-uncorrelated generative variables, we can obtain invariance over $G_t$ by balancing the training data distributions and make those generative variables statistically independent with all the other generative variables and the labels.

Combining Theorem 3,4,5, leads to an integrated solution summarized as Algorithm \ref{alg} (abbreviated as \emph{InvarTG}). \emph{InvarTG} provides a practical procedure to lower generalization bound by obtaining invariance over the task-uncorrelated generative variables.
\begin{algorithm}[htb]
\caption{ Obtaining Invariance over Task-uncorrelated Generative variables (InvarTG) }
\label{alg}
\begin{algorithmic} [1]
\REQUIRE ~~\\ 
- the raw training set, $S$;\\
- the set of task-uncorrelated generative variables, $G_U=\{G_U^{(i)}\}_{i=1...N}$;\\
- the target threshold for imbalanced generative variables, $thrd$;\\
- the training algorithm, $Train$;\\
- a method to manipulate the value of $G_U$ in training samples to get an augmented new set, $Balance$;\\

\ENSURE ~~\\ 
- a model with lower generalization bound, $h^*$;

$\ $
\REPEAT
\STATE Choosing a target task-uncorrelated generative variable which have the minimum $H_{S}(Y|G_U^{(i)})$, $U_T\leftarrow\arg\min\limits_{G_U^{(i)}}{H_{S}(Y|G_U^{(i)})}$;
\label{a2}
\STATE  Using a method to modify $U_T$ in training set $S$, so that $U_T$ will be statistically independent with other variables, $S\leftarrow Balance(S,U_T)$;
\label{a3}
\STATE  Deleting $U_T$ from $G_U$, $G_U\leftarrow \{U| U \in G_U \bigwedge U\neq U_T \}$;
\UNTIL{$H_{S}(Y|G_U^{(i)})> thrd$}
\STATE  Training with the new training set and get the model, $h^*\leftarrow Train(S')$;
\label{a4}
\RETURN $h^*$; 
\end{algorithmic}
\end{algorithm}

\subsection{Explanation on the Data Augmentation}
Data augmentation is devoted to optimizing data distribution and has achieved improved generalization performance in many scenarios \cite{Zhang2018mixup}\cite{zhong2017random}. We can find that data augmentation solutions share many similarities with the above introduced theorems to obtain invariance. Actually, data augmentation can serve as one typical way to instantiate the \emph{Balance} method in \emph{InvarTG} algorithm. In this subsection, we will use data augmentation as an example to help better understand the proposed theorems as well as provide a theoretical explanation to the effectiveness of data augmentation.

The effectiveness of data augmentation in improving generalization performance is recognized not only due to the increased training data size: with the same size of training data, different data augmentation solutions will obtain different generalization performances\cite{Ekin2018Autoaugment}. This critically contradicts with the traditional generalization theories whose generalization bound only concerns about the number of training data. Moreover,  many data augmentation solutions like color jitter, random-erasing, mix-up add new samples rarely appearing in real scenarios, which obviously violates the IID assumption. This further invalids the traditional generalization theories in explaining the effectiveness of data augmentation.

On the contrary, while most data augmentation solutions inevitably change the training data distributions, they maintain the same marginal distribution over the task-correlated generative variables. Taking the dog image classification task as an example, color jitter changes the object colors (i.e., task-uncorrelated generative variable), but the intrinsic features of dogs such as the appearance of legs, w. or w/o. tail (i.e., task-correlated generative variable) remain unchanged. The compatibility of the ITID assumption makes it possible to explain the effectiveness of data augmentation by employing the proposed generalization theory.

According to the proposed task-related generalization theory, in addition to increasing training data size, data augmentation also contributes to the improved generalization by balancing the data distribution regarding generative variables. For example, traditional data augmentation solutions of random cropping, random rotation, and color jitter respectively balance the data distribution regarding existing regular task-uncorrelated generative variables of position, angle, and color. In this way, the influence of those task-uncorrelated generative variables on model prediction is reduced so as to lower the overall generalization bound and improve the generalization performance.

The proposed theory is also helpful in explaining many new advanced data augmentation solutions. These new augmentation solutions, such as random erasing\cite{zhong2017random}, mix-up\cite{Zhang2018mixup}, RICAP\cite{takahashi2018ricap}, usually generate very different samples from the original training set, yet obtain better generalization performance than traditional augmentation solutions. According to the proposed generalization theory, these solutions actually introduce new task-uncorrelated generative variables to help better balance the existing ones. For example, occlusion is one regular task-uncorrelated generative variables existing in many object classification tasks, which is difficult to be balanced via traditional data augmentation. To reduce its influence on model prediction, random erasing introduces additional rectangle noise to randomly erase a part of the input image, which decreases the dependence between occlusion and label distributions. The question is, since rectangle noise rarely exists in original data and can be viewed as a new task-uncorrelated generative variable,  will it increase the total influence $\gamma$? Note that when adding the new of task-uncorrelated generative variables, their parameters (like the area and aspect ratio in random erasing~\footnote{\small{More evidence and discussions on this are provided in the experimental validation - Section 5.2.}}) are sampled independently with the inputs and the labels. According to Theorem 5, optimal models are guaranteed to learn invariance to these independently distributed task-uncorrelated generative variables. This explains the reason that these data augmentation solutions add very different training samples but lead to better generalization performance.

\section{Experimental Validation}
\label{exp}
The application of the proposed generalization theory involves with three key conclusions: (1) Influence measure. Theorem \ref{thm5} proves that the conditional entropy in the training dataset $H_S(Y|G_U^{(i)})$ can be used to measure the influence of $G_U^{(i)}$ on model predictions. (2) Invariance acquisition.  Theorem \ref{thm6} proves to obtain model invariance over generative variables by balancing the distribution of the training set. (3) Generalization improvement. Theorem \ref{thm3} and Theorem \ref{thm4} together prove that model will have a lower generalization bound by enhancing invariance over task-uncorrelated generative variables.  In this section, we design experiments on both toy data and a real-world image dataset to validate these conclusions as well as demonstrate the effectiveness of the proposed theory in improving generalization performance.

\subsection{Validation on Toy Data}
\label{s5}
\subsubsection{Experimental Settings}
\label{tes}
We design a toy task of binary classification (label $c=\{0,1\}$). Each class $c$ consists of 5,000 instances $\mathbf{x}_c^{(i)}\in \mathbb{R}^{20}$, with half as training data, half as testing data. The goal is to learn a hypothesis $h\in H: R^{20}\rightarrow \{0,1\}$. We first construct a dataset following the IID assumption: training and testing data of class $c$ are sampled from the identical multivariate Gaussian distribution  $\mathbf{x}_c^{(i)} \sim \mathcal{N}(\mathbf{\mu}_c,\mathbf{\sum}_c); i=1,\cdots,5,000$. $\mathbf{\mu}_c\in \mathbb{R}^{20},\mathbf{\sum}_c\in \mathbb{R}^{20\times 20}$ denote the fixed mean vector and semi-positive covariance matrix for class $c$ , respectively.

To validate the proposed generalization theory on ITID assumption, we need first to instantiate generative variables. Same as Example 1, we assume identity mapping as the generating function: $\phi(g)=g$, i.e., the input features have one-to-one mapping with the generative variables. To distinguish between task-correlated and -uncorrelated generative variables, we restrict the training and testing data to sample from identical input distributions corresponding to task-correlated generative variables, and from different input distributions corresponding to task-uncorrelated ones. Specifically, among the 20-d input features, the first 10-d are assumed as task-correlated, and the last 10-d as task-uncorrelated. To realize this, for each class $c$, the training data are sampled following the same process as under the IID assumption: $\mathbf{x}_c^{(i)} \sim \mathcal{N}(\mathbf{\mu}_c,\mathbf{\sum}_c); i=1,\cdots,2,500$. When generating the testing data, for each testing instance $\mathbf{x}_c^{(i)}, i=2,501, \cdots, 5,000$, we first construct a new Gaussian distribution $\mathcal{N}(\mathbf{\mu}_c^{(i)},\mathbf{\sum}_c^{(i)})$ by substituting the last 10-d of  $\mathbf{\mu}_c$ with random vector to generate a new mean vector $\mathbf{\mu}_c^{(i)}$, and substituting the upper-right, bottom-left, bottom-right $10\times 10$ submatrices of  $\mathbf{\sum}_c$ to generate a new covariance matrix $\mathbf{\sum}_c^{(i)}$. The testing instance is then sampled from the new distribution: $\mathbf{x}_c^{(i)} \sim \mathcal{N}(\mathbf{\mu}_c^{(i)},\mathbf{\sum}_c^{(i)}); i=2,501,\cdots,5,000$. This process is illustrated in Fig.\ref{fig3}. In this way,  we construct a toy dataset satisfying the proposed ITID assumption: the testing instances are sampled from different distributions but maintain the same marginal distribution over the task-correlated generative variables.

Other experimental settings are as follows. Regarding the hypothesis, the models to be learned in toy experiments are fixed as generalized linear models with sigmoid function. Regarding the training process, the optimization target is empirical risk minimization with binary cross-entropy loss, and the training algorithm is stochastic gradient descent with $0.01$ learning rate and $0.9$ momentum. Besides, the batch size is 256, and all the models are trained with 100 epochs. To get a stable result, we randomly generate 100 different toy datasets with the above generating process. The results reported in the following validation experiments are averaged over the 100 toy datasets..

\begin{figure}[!t]
\centering
\includegraphics[width=3in]{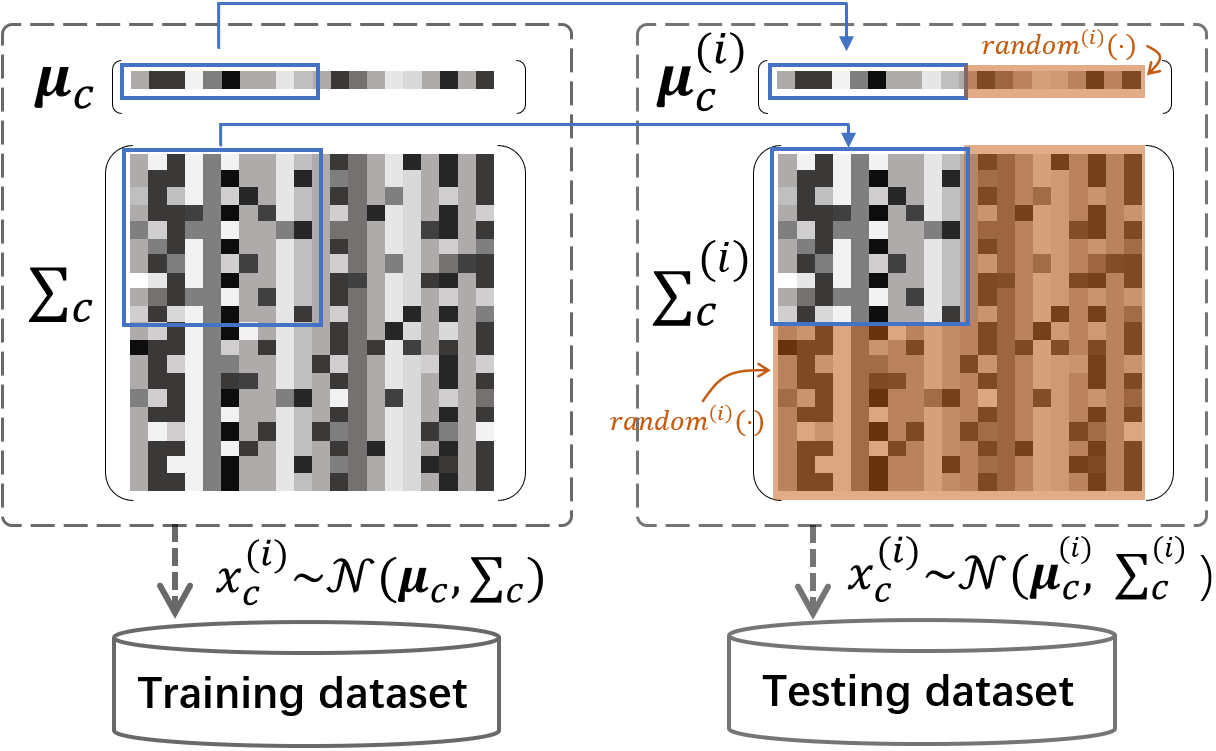}
\caption{Illustration to generate a toy dataset satisfying ITID assumption.}
\label{fig3}
\end{figure}
\subsubsection{Validation of Influence Measure}
Given task-uncorrelated generative variable set $G_U$, the target is to validate whether $H_S(Y|G_U^{(i)})$ is a good measure to estimate the real influence of $G_U^{(i)}\in G_U$ on the model predictions.
In the above toy task, for each of the last 10-d dimensions,  $H_S(Y|G_U^{(i)})$ can be readily calculated from the training dataset.
Regarding the real influence, it is recognized that the weights in the generalized linear models represent the significance of the corresponding input features on model prediction. Therefore, following the training process introduced in the experimental setting, we learned a model $h_{original}$ and use the absolute weight values corresponding to the last 10-d input features as the ground-truth influence.
For both the values of $H_S(Y|G_U^{(i)})$ and the absolute weights in $h_{original}$, we order the task-uncorrelated generative variables from the highest to the lowest to get the estimated rank and ground-truth influence rank respectively.

\begin{figure*}[!t]
\centering
\subfloat[]{\includegraphics[width=2.2in]{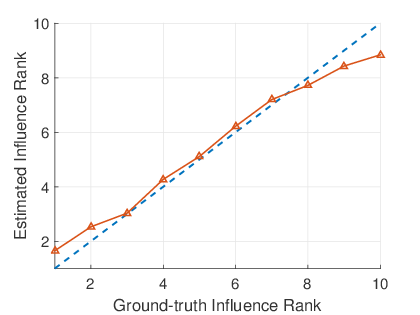}\label{toy-a}}
\hfil
\subfloat[]{\includegraphics[width=2.2in]{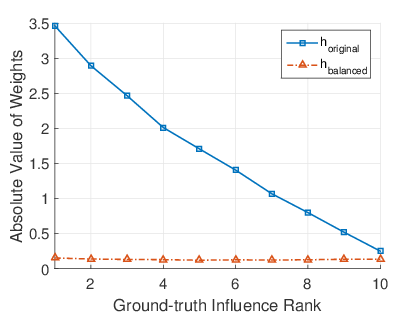}\label{toy-b}}
\hfil
\subfloat[]{\includegraphics[width=2.2in]{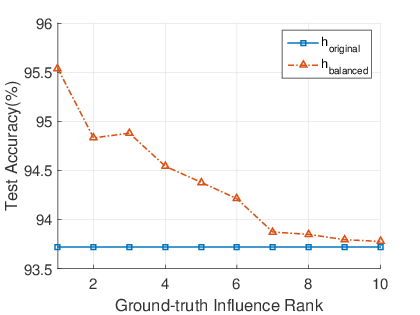}\label{toy-c}}

\caption{Validation experimental results averaged over 100 runs: (a) the estimated influence ranks  by $H_S(Y|G_U^{(i)})$ (y-axis) v.s. the ground-truth influence ranks by the absolute weight values; (b) the changes of the absolute weight values before and after \emph{Balance} operation; (c) test accuracy before and after \emph{Balance} operation.}
\label{toy}
\end{figure*}

Fig.3(a) compares the estimated ranks (y-axis) to the ground-truth influence ranks (x-axis). It can be seen the estimated influence rank curve is close to the reference line $y=x$, which means that $H_S(Y|G_U^{(i)})$ accurately approximates the real influence of task-uncorrelated generative variables.
Note that the estimated-influence rank slightly deviates from the reference line when the influence rank is lower than $8$. This result is possibly due to the very small influence of these input features (the influence weights are generally between $[0,0.8]$ for rank $8\sim 10$), which makes their relative ranks unstable in different toy datasets.

\subsubsection{Validation of Invariance Acquisition}
\label{above}
According to Theorem \ref{thm6}, the models can obtain invariance over task-uncorrelated generative variables which are statistically independent with the other generative variables and labels.
To validate this, for each of the $10$ task-uncorrelated generative variables, we modify the training data distributions to make it statistically independent, and then examine its influence change before and after the modification.
Specifically, given task-uncorrelated generative variables $G_U^{(i)}$, we implement the \emph{Balance} operation by substituting the corresponding input dimension with a random value sampled from the uniform distribution over $[0,1]$. The new model is trained on this modified training dataset to derive model $h_{balanced}$. To avoid the intervene between different generative variables, we conduct the above process for each individual task-uncorrelated generative variable and record its weight value in model $h_{balanced}$.

Fig.3(b) shows the absolute weight values learned in models before and after balancing modification (annotated as $h_{original}$ and $h_{balanced}$, respectively). To guarantee a consistent experimental discussion, we keep the x-axis as the same in Fig.\ref{toy-a}. It is shown that in spite of the very different weight values in the original model, all task-uncorrelated generative variables obtain the model weight of almost 0 after balanced modification. This validates Theorem \ref{thm6} that  the \emph{Balance} operation can help model acquire invariance over target generative variables.

\subsubsection{Validation of Generalization Improvement}

According to Theorem \ref{thm3} and Theorem \ref{thm4}, a model will have lower generalization bound when enhancing invariance over task-uncorrelated generative variables.
To validate this, following the same settings in the above experiments, we examine and compare the performance of the trained $h_{original}$ and $h_{balanced}$ in the testing dataset.
The result is shown in Fig.3(c). It can be seen that after balancing the task-uncorrelated generative variables, the learned models achieve consistently improved testing performance. Moreover, the more influence the task-uncorrelated generative variable (with higher influence rank), the more improvement $h_{balanced}$ can achieve after enhancing the variance over it. This further demonstrates the theoretically proved relationship between invariance enhancement and generalization performance.

\subsection{Validation on Real-world Dataset}
\label{s6}
\subsubsection{Experimental Settings}
\label{es}
In addition to the validations in the toy dataset, we also conducted experiments on real-world image classification tasks with deep learning models.
Specifically,  we select the CIFAR-10 dataset\cite{krizhevsky2009learning}, consisting of 60,000 $32\times32$ color images in 10 classes, with 6,000 images per class. Among the 60,000 images, 50,000 constitutes the training set and the rest 10,000 constitutes the testing set.
To avoid the biased conclusion from certain model, we report experimental results using three commonly-used CNN structures: VGG-16\cite{simonyan2014very}, ResNet-20\cite{he2016deep}, DenseNet-40-12-BC\cite{huang2017densely}.
The model hyperparameters are empirically set: the batch size is 128,  the weight decay is $1e^{-4}$, the learning rate is 0.1 at the beginning and changes to 0.01/0.001 at 150/225 epochs respectively.

\begin{figure*}[!t]
\centering
\subfloat[VGG]{\includegraphics[width=2in]{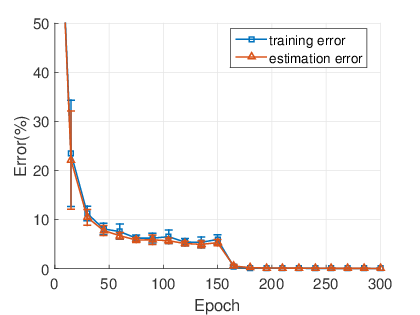}\label{confidence-a}}
\hfil
\subfloat[ResNet]{\includegraphics[width=2in]{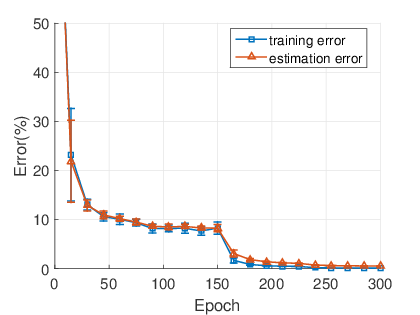}\label{confidence-b}}
\hfil
\subfloat[DenseNet]{\includegraphics[width=2in]{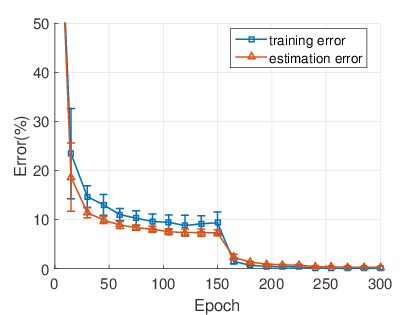}\label{confidence-c}}
\caption{Validation of Corollary \ref{cor2}: the curve of true training errors and errors estimated by the output values according to Eqn.\eqref{eq-cor2}.}
\label{confidence}
\end{figure*}
\subsubsection{Validation of Influence Measure}
\label{vt5}
In practical tasks, it is usually not easy to explicitly determine all generative variables. We turn to validate the following corollary from Theorem \ref{thm5} to indirectly validate the influence measure conclusion.

\begin{cor}\label{cor2}
Given the optimal hypothesis $h'$ satisfying Theorem \ref{thm5}, we have the training error
\begin{equation}\label{eq-cor2}
L_{S}(h')=1-\mathbb{E}_{x\in S}\max_{y\in \mathcal{Y}}{Q'_y(x)}
\end{equation}
where $\mathbb{E}_{x\in S}$ is the expectation over $S$.
\end{cor}
The proof of Corollary \ref{cor2} can be seen in Appendix-F.
Corollary \ref{cor2} shows that if Theorem \ref{thm5} is satisfied, the mean of the maximum output values can accurately estimate the training accuracy.
Fig.\ref{confidence} compares the curves of the estimated training errors calculated according to Corollary \ref{cor2} with the curves of true training errors during the training process. All three examined model structures show consistently similar tendencies between these two curves. Note that training error is regularly calculated according to model predictions instead of specific output values. Therefore, with the observed close relationship  between training errors and maximum output values, we can validate Corollary \ref{cor2} and also Theorem \ref{thm5}  in real-world tasks to some extent.

The experimental results are shown in Fig.\ref{confidence}. One interesting observation is that the estimated error curve and true error curve are close to each other during the whole training process. According to Corollary \ref{cor2}, we only expect the closeness when convergence after training.  One possible reason is that the generative variable set that model determines dynamically changes during the training process. Given temporal determining generative variable set  $G_l$,  CNN models can quickly obtain minimum cross-entropy losses to obtain an optimal hypothesis that is strict-$G_l$-determining. Therefore, the continuous closeness between the two curves demonstrates the updated determining/invariant generative variable set of the optimal hypothesis.

This explanation can also be supported by the experimental results in recent CNN visualization studies\cite{Zeiler2013Visualizing}. It is shown that the CNN convolutional filters converge first on the lower layers and then to the higher layers. As the converged filters are only sensitive to some specific patters, we can roughly regard these specific patterns as generative variables. Therefore, the converged lower layers actually first learn invariance over certain generative variables and the higher layers then optimize this invariant generative variable set during the training process.  In this way, deep models can gradually obtain invariance over as many task-uncorrelated generative variables as possible to improve generalization performance.

\begin{table*}[!t]
\renewcommand{\arraystretch}{1.3}
\caption{Distribution setting of $P_1$: the distribution of task-uncorrelated generative variables is depend on the label.}
\label{ab}
\centering
\begin{tabular}{|c|c|c||c|c|c|}
\hline
labels    &$G_U^{(1)}: [a,b]$  & $G_U^{(2)}:[a,b]$ & labels & $G_U^{(1)}: [a,b]$ & $G_U^{(2)}: [a,b]$  \\
\hline
0   &$[0,\frac{1}{3}]$   &$[0,\frac{1}{3}]$    &5  &$[\frac{1}{3},\frac{2}{3}]$ &$[\frac{2}{3},1]$ \\
\hline
1   &$[0,\frac{1}{3}]$   &$[\frac{1}{3},\frac{2}{3}]$    &6  &$[\frac{2}{3},1]$ &$[0,\frac{1}{3}]$ \\
\hline
2   &$[0,\frac{1}{3}]$   &$[\frac{2}{3},1]$    &7  &$[\frac{2}{3},1]$ &$[\frac{1}{3},\frac{2}{3}]$ \\
\hline
3   &$[\frac{1}{3},\frac{2}{3}]$   &$[0,\frac{1}{3}]$    &8  &$[\frac{2}{3},1]$ &$[\frac{2}{3},1]$ \\
\hline
4   &$[\frac{1}{3},\frac{2}{3}]$   &$[\frac{1}{3},\frac{2}{3}]$    &9  &$[0,0]$ &$[0,0]$ \\
\hline

\end{tabular}
\end{table*}

\begin{figure}[!t]
\centering
\includegraphics[width=2.5in]{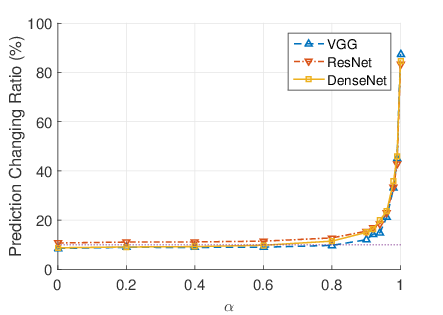}
\caption{Validation of invariance acquisition: the influence of dependent task-uncorrelated generative variable (x-axis) on model prediction changing (y-axis)}
\label{influence-a}
\end{figure}

\begin{figure}[!t]
\centering
\includegraphics[width=2.5in]{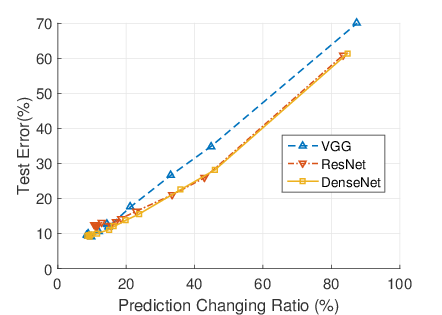}
\caption{Validation of generalization improvement: the test errors (y-axis) increase as the influence of dependent task-uncorrelated generative variable (x-axis) increases.}
\label{influence-b}
\end{figure}

\begin{figure*}[!t]
\centering
\subfloat[Distribution]{\includegraphics[width=1.7in]{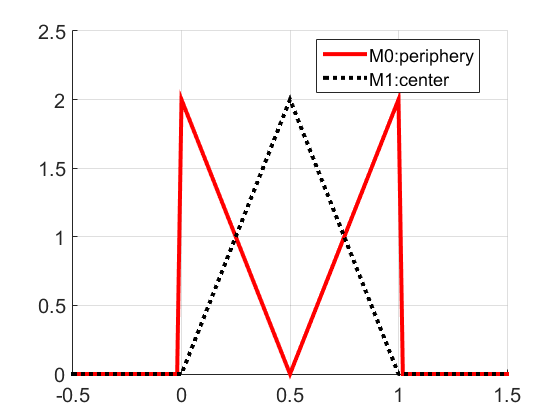}\label{improve-dis}}
\hfil
\subfloat[VGG]{\includegraphics[width=1.8in]{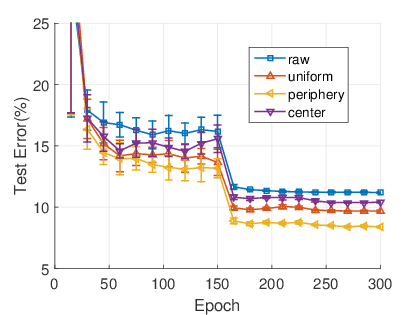}\label{improve-a}}
\hfil
\subfloat[ResNet]{\includegraphics[width=1.8in]{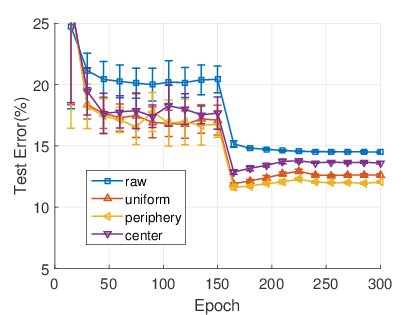}\label{improve-b}}
\hfil
\subfloat[DenseNet]{\includegraphics[width=1.8in]{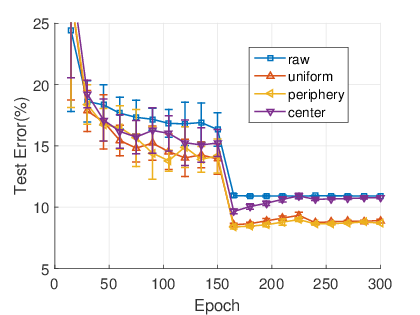}\label{improve-c}}
\caption{Test errors of different random-erasing settings.}
\label{improve}
\end{figure*}

\subsubsection{Validation of Invariance Acquisition}
\label{above2}
Theorem \ref{thm6} proves that the optimal models will have invariance over the statistically independent task-uncorrelated generative variables.
We consider employing a data augmentation method to introduce new task-uncorrelated generative variables whose parameters can be obtained and manipulated easily.
In this way, we can control the parameter distributions as statistically independent or dependent on the labels $Y$, to validate the effect of statistical independence on model invariance acquisition.

We use random-erasing\cite{zhong2017random} as the example data augmentation method, and choose its two parameters as the task-uncorrelated generative variables to be modified: the area and the aspect ratio of the erasing rectangles, which are denoted as $G_U^{(1)}$ and $G_U^{(2)}$.
Two reference distribution settings $P_0, P_1$ are designed to respectively represent the case of full independence and dependence.
 As for $P_0$, $G_U^{(1)}$ and $G_U^{(2)}$ are independent both with the labels $Y$ and with each other:  $G_U^{(1)} \sim Uniform(0,1)$, $G_U^{(2)}\sim Uniform(0,1)$.
As for $P_1$,  $G_U^{(1)}$ and $G_U^{(2)}$ still follow the uniform distribution $Uniform(a,b)$, but the distribution parameter $[a,b]$ is dependent on the labels $Y$. Table \ref{ab} lists the allocated values of $[a,b]$ for different labels, e.g., when sample label is 0, we have $G_U^{(1)}  \sim Uniform(0,\frac{1}{3})$, $G_U^{(2)} \sim Uniform(0,\frac{1}{3})$. To examine the effect of different levels of dependence, we define the final distribution as the mixture of the two reference settings:
\begin{equation}
P_{\alpha}=\left\{
\begin{aligned}
P_0 &,\ & with\ probability&\ & 1-\alpha  \\
P_1 &,\ & with\ probability&\ & \alpha
\end{aligned}
\right.
\end{equation}
where $\alpha \in [0,1]$ controls the level of dependence.

Regarding the measure of model invariance, we examine how much the new models learned from augmented training dataset change the predicted labels. Specifically,  we randomly apply random-erasing to the original training set where the augmented samples follow $P_{\alpha}$, and employ the learned new model to output predictions for the testing instances. The above process repeats 100 times and we calculate the \emph{prediction changing ratio} as how many of the testing instances obtain predicted labels different from the original model.
The results are shown in Fig.\ref{influence-a}. It is shown that model tends to change its prediction when $\alpha$ increases for all three examined CNN structures, indicating that highly dependent task-uncorrelated generative variables will lead to higher influence on the model prediction. This validates our conclusion for invariance acquisition that statistical independence is desired for the task-uncorrelated generative variables to ensure the model invariance.

Two more discussions on Fig.\ref{influence-a}:
(1) When $\alpha=0$, we can find that there is still about $10\%$ chance to change the model predictions, rather than $0\%$ indicating complete model invariance.
Such a phenomenon might be because that while the added two new generative variables are controlled to be statistically independent and have no influence to model output, as discussed in Section 4.4, random-erasing yet modifies the distribution of some existing generative variables like occlusion so as to change the model prediction.
(2) The prediction changing ratio non-trivially increases only when $\alpha>0.8$, i.e., the task-uncorrelated generative variables are heavily dependent on the labels. This indicates that while statistical independence theoretically guarantees full model invariance, not-too-heavy dependence is adequate to maintain a stable model prediction to some extent. This observation is very important to practical implementations: in real-world tasks, it is very difficult to balance a task-uncorrelated generative variable to be strictly independent, the goal can thus be simplified to weaken the dependence of task-uncorrelated generative variables. The effectiveness of random-erasing can also be explained from this perspective: although the augmented dataset cannot guarantee full independence between the task-uncorrelated generative variable of occlusion and the labels, the introduced erasing masks succeed to weaken this dependence and contribute to observable performance improvement.

\subsubsection{Validation of Generalization Improvement}
The above experiment derives different levels of invariance by adjusting $P_{\alpha}$. Following the same setting, we examine the relationship between generalization performance and model invariance over the task-uncorrelated generative variables. Fig.\ref{influence-b} shows the curve of test error versus prediction changing ratio. It is shown that the generalization performance is positively correlated with the model invariance, which validates the generalization improvement conclusion in real-world tasks. Moreover, among the three examined CNN structures, we observe that ResNet and DenseNet achieve better generalization performance than VGG when task-uncorrelated generative variables strongly influence the model prediction (i.e., prediction changing ratio$>20\%$). One possible reason is that ResNet and DenseNet are equipped with stronger mechanisms to address the dependent task-uncorrelated generative variables, which also explains their recognized superior performance in many tasks. We will be working in future studies how to apply the proposed generalization theory from a model design perspective to obtain invariance over the task-uncorrelated generative variables.

\subsection{Improvement for Data Augmentation}
Note that the proposed generalization theory is based on the ITID assumption, which requires a unique marginal distribution over the task-correlated generative variables.
However, many data augmentation solutions like random erasing are potential to change the distribution of task-correlated generative variables, which decrease the generalization performance.
In this subsection, we conduct a simple experiment to reduce this potential to affect task-correlated generative variables and examine its relationship to the generalization performance.

By examining the Cifar-10 dataset, we find that the main objects usually appear in the center of the images.
This inspires us to modify the distribution of erasing positions to appear beyond the center and help avoid affecting the task-correlated generative variables.
Specifically, two parameters control the erasing position in random erasing, horizontal and vertical coordinates $G_U^{(3)}$ and $G_U^{(4)}$\footnote{Horizontal and vertical coordinates are represented as the ratio to the length and width of image: $G_U^{(3)}, G_U^{(4)}\in [0,1]$}.
In traditional random erasing, the erasing position evenly distributes over the whole image:  $G_U^{(3)},G_U^{(4)}\sim Uniform(0,1)$.
To realize the idea of not affecting the task-correlated generative variables that usually appear near the center, we design a new distribution $M_0$ with probability density function $p_{m0}(x)=4|x-0.5|$.
To facilitate the comparison, we also design another distribution $M_1$ with probability density function $p_{m1}(x)=2-4|x-0.5|$. The probability density curves of $M_0$ and $M_1$ are shown in Fig.\ref{improve}-a. It is obvious that the erasing following $M_0$ will appear more at the periphery, while erasing the following $M_1$ will appear more at the center of images and have a higher probability to change the marginal distribution of task-correlated generative variables.

Fig.\ref{improve}(b)-(d) show the test error in four experimental settings: (1) \emph{raw}, the original training set without data augmentation; (2) \emph{uniform}, the random erasing method with traditional uniform distribution on the erasing position; (3) \emph{periphery}, the random erasing method with position parameters following $M_0$; and (4) \emph{center}, the random erasing method with position parameters following $M_1$.  As can be seen, when restricting the erasing positions near the periphery area, the performance of \emph{periphery} consistently outperforms the other random erasing settings. While, \emph{center} achieves inferior performance than \emph{uniform}, further validating that affecting the distribution of task-correlated generative variables will violate the ITID assumption and tends to deteriorate the generalization performance.

As a result, we summarize two guidelines from the proposed generalization theory in data augmentation design and other methods optimizing data distributions: (1) balancing the task-uncorrelated generative variables to reduce their dependence with other generative variables and the label; (2) not affecting the distribution of task-correlated generative variables to maintain the task property.

\section{Conclusion and Future Work}
\label{s7}
 In this work, we introduce intermediate generative variables into the data generating process, which considers the correlation with solved tasks to establish a new ITID assumption. The derived generalization bound based on the ITID assumption shows that gaining a model with invariant task-uncorrelated generative variables is very important for good generalization. Applying our theory can get an algorithm to improve the generalization performance, and interpret the mechanism of the data augmentation. The experiments on a toy dataset and Cifar-10 dataset show that our theory can improve the generalization performance in practice.

 The future work can be divided into theoretical and application directions. The future theoretical directions include: (1) Obtaining invariance to task-uncorrelated generative variables shares many commonnesses with the information bottleneck theory \cite{tishby2015deep}\cite{shwartz2017opening}. It will be interesting to analyze the layer-wise transformation in deep neural networks from the perspective of hypothesis determination update on generative variables. (2)  In this paper, the generalization bound is derived based on Gibbs classifiers. Applying Bayesian classifiers\cite{Mcallester1999PAC} has the potential to derive a tighter generalization bound. (3) The current generalization theory is formalized and applied in the context of single-type input, e.g., image in this paper. It is worth extending it to address learning tasks with multi-type input, .e.g., recommendation tasks with input $X={user, item}$ and label $Y={rate}$.

 The future application directions include: (1)  Obtaining invariance to task-uncorrelated generative variables provides theoretically guaranteed objective to model data augmentation design as an optimization problem. We will also instantiate the task-uncorrelated and -correlated generative variables in real-world tasks and dataset by combining with the existing generative models (e.g., Generative Adversarial Network).  (2) Regarding the way of modifying data distribution for invariance enhancement, data sampling solutions can also be explored in addition to the data augmentation algorithm introduced in this paper. It is noted that the focal loss can be viewed as implicitly increasing invariance to task-uncorrelated generative variables via data sampling. (3) From the perspective of generative variables, the adversarial example phenomenon can be explained that the adversarial perturbation changes task-uncorrelated generative variables and model is heavily influenced by those variables. Therefore, we can apply the new generalization theory to improve adversarial robustness by either modifying data distribution or developing models determining more on task-correlated generative variables.

\ifCLASSOPTIONcaptionsoff
  \newpage
\fi

\bibliographystyle{IEEEtran}
\bibliography{IEEEabrv,mybib}

\begin{thebibliography}{10}
\providecommand{\url}[1]{#1}
\csname url@samestyle\endcsname
\providecommand{\newblock}{\relax}
\providecommand{\bibinfo}[2]{#2}
\providecommand{\BIBentrySTDinterwordspacing}{\spaceskip=0pt\relax}
\providecommand{\BIBentryALTinterwordstretchfactor}{4}
\providecommand{\BIBentryALTinterwordspacing}{\spaceskip=\fontdimen2\font plus
\BIBentryALTinterwordstretchfactor\fontdimen3\font minus
  \fontdimen4\font\relax}
\providecommand{\BIBforeignlanguage}[2]{{%
\expandafter\ifx\csname l@#1\endcsname\relax
\typeout{** WARNING: IEEEtran.bst: No hyphenation pattern has been}%
\typeout{** loaded for the language `#1'. Using the pattern for}%
\typeout{** the default language instead.}%
\else
\language=\csname l@#1\endcsname
\fi
#2}}
\providecommand{\BIBdecl}{\relax}
\BIBdecl

\bibitem{vapnik2015uniform}
V.~N. Vapnik and A.~Y. Chervonenkis, ``On the uniform convergence of relative
  frequencies of events to their probabilities,'' in \emph{Measures of
  complexity}.\hskip 1em plus 0.5em minus 0.4em\relax Springer, 2015, pp.
  11--30.

\bibitem{bartlett2002rademacher}
P.~L. Bartlett and S.~Mendelson, ``Rademacher and gaussian complexities: Risk
  bounds and structural results,'' \emph{Journal of Machine Learning Research},
  vol.~3, no. Nov, pp. 463--482, 2002.

\bibitem{mukherjee2006learning}
S.~Mukherjee, P.~Niyogi, T.~Poggio, and R.~Rifkin, ``Learning theory: stability
  is sufficient for generalization and necessary and sufficient for consistency
  of empirical risk minimization,'' \emph{Advances in Computational
  Mathematics}, vol.~25, no. 1-3, pp. 161--193, 2006.

\bibitem{xu2012robustness}
H.~Xu and S.~Mannor, ``Robustness and generalization,'' \emph{Machine
  learning}, vol.~86, no.~3, pp. 391--423, 2012.

\bibitem{bartlett1998sample}
P.~L. Bartlett, ``The sample complexity of pattern classification with neural
  networks: the size of the weights is more important than the size of the
  network,'' \emph{IEEE transactions on Information Theory}, vol.~44, no.~2,
  pp. 525--536, 1998.

\bibitem{hardt2015train}
M.~Hardt, B.~Recht, and Y.~Singer, ``Train faster, generalize better: Stability
  of stochastic gradient descent,'' \emph{arXiv preprint arXiv:1509.01240},
  2015.

\bibitem{simonyan2014very}
K.~Simonyan and A.~Zisserman, ``Very deep convolutional networks for
  large-scale image recognition,'' \emph{arXiv preprint arXiv:1409.1556}, 2014.

\bibitem{zhang2016understanding}
C.~Zhang, S.~Bengio, M.~Hardt, B.~Recht, and O.~Vinyals, ``Understanding deep
  learning requires rethinking generalization,'' \emph{arXiv preprint
  arXiv:1611.03530}, 2016.

\bibitem{wolpert1995no}
D.~H. Wolpert, W.~G. Macready \emph{et~al.}, ``No free lunch theorems for
  search,'' Technical Report SFI-TR-95-02-010, Santa Fe Institute, Tech. Rep.,
  1995.

\bibitem{arpit2017closer}
D.~Arpit, S.~Jastrz?bski, N.~Ballas, D.~Krueger, E.~Bengio, M.~S. Kanwal,
  T.~Maharaj, A.~Fischer, A.~Courville, and Y.~Bengio, ``A closer look at
  memorization in deep networks,'' in \emph{Proceedings of the 34th
  International Conference on Machine Learning-Volume 70}.\hskip 1em plus 0.5em
  minus 0.4em\relax JMLR. org, 2017, pp. 233--242.

\bibitem{Lin2017Focal}
R.~G. K. H. P.~D. Lin Tsung-Yi, Priya~Goyal, ``Focal loss for dense object
  detection,'' 2017, pp. 2980--2988.

\bibitem{swaminathan2017new}
M.~Swaminathan, P.~K. Yadav, O.~Piloto, T.~Sj{\"o}blom, and I.~Cheong, ``A new
  distance measure for non-identical data with application to image
  classification,'' \emph{Pattern Recognition}, vol.~63, pp. 384--396, 2017.

\bibitem{shi2017beyond}
Y.~Shi, W.~Li, Y.~Gao, L.~Cao, and D.~Shen, ``Beyond iid: Learning to combine
  non-iid metrics for vision tasks,'' in \emph{Thirty-First AAAI Conference on
  Artificial Intelligence}, 2017.

\bibitem{Zhang2018mixup}
H.~Zhang, M.~Cisse, Y.~N. Dauphin, and D.~Lopez-Paz, ``mixup: Beyond empirical
  risk minimization,'' 2018.

\bibitem{zhong2017random}
Z.~Zhong, L.~Zheng, G.~Kang, S.~Li, and Y.~Yang, ``Random erasing data
  augmentation,'' \emph{arXiv preprint arXiv:1708.04896}, 2017.

\bibitem{Shalev2014Understanding}
S.~Shalev-Shwartz and S.~Ben-David, \emph{Understanding Machine Learning: From
  Theory to Algorithms}, 2014.

\bibitem{valiant1984theory}
L.~G. Valiant, ``A theory of the learnable,'' in \emph{Proceedings of the
  sixteenth annual ACM symposium on Theory of computing}.\hskip 1em plus 0.5em
  minus 0.4em\relax ACM, 1984, pp. 436--445.

\bibitem{haussler2018decision}
D.~Haussler, ``Decision theoretic generalizations of the pac model for neural
  net and other learning applications,'' in \emph{The Mathematics of
  Generalization}.\hskip 1em plus 0.5em minus 0.4em\relax CRC Press, 2018, pp.
  37--116.

\bibitem{yu1994rates}
B.~Yu, ``Rates of convergence for empirical processes of stationary mixing
  sequences,'' \emph{The Annals of Probability}, pp. 94--116, 1994.

\bibitem{mohri2008stability}
M.~Mohri and A.~Rostamizadeh, ``Stability bounds for non-iid processes,'' in
  \emph{Advances in Neural Information Processing Systems}, 2008, pp.
  1025--1032.

\bibitem{mohri2009rademacher}
------, ``Rademacher complexity bounds for non-iid processes,'' in
  \emph{Advances in Neural Information Processing Systems}, 2009, pp.
  1097--1104.

\bibitem{bartlett1992learning}
P.~L. Bartlett, ``Learning with a slowly changing distribution,'' in
  \emph{Proceedings of the fifth annual workshop on Computational learning
  theory}.\hskip 1em plus 0.5em minus 0.4em\relax ACM, 1992, pp. 243--252.

\bibitem{helmbold1994tracking}
D.~P. Helmbold and P.~M. Long, ``Tracking drifting concepts by minimizing
  disagreements,'' \emph{Machine learning}, vol.~14, no.~1, pp. 27--45, 1994.

\bibitem{bartlett2000learning}
P.~L. Bartlett, S.~Ben-David, and S.~R. Kulkarni, ``Learning changing concepts
  by exploiting the structure of change,'' \emph{Machine Learning}, vol.~41,
  no.~2, pp. 153--174, 2000.

\bibitem{mohri2012new}
M.~Mohri and A.~M. Medina, ``New analysis and algorithm for learning with
  drifting distributions,'' in \emph{International Conference on Algorithmic
  Learning Theory}.\hskip 1em plus 0.5em minus 0.4em\relax Springer, 2012, pp.
  124--138.

\bibitem{ralaivola2010chromatic}
L.~Ralaivola, M.~Szafranski, and G.~Stempfel, ``Chromatic pac-bayes bounds for
  non-iid data: Applications to ranking and stationary $\beta$-mixing
  processes,'' \emph{Journal of Machine Learning Research}, vol.~11, no. Jul,
  pp. 1927--1956, 2010.

\bibitem{barve1997complexity}
R.~D. Barve and P.~M. Long, ``On the complexity of learning from drifting
  distributions,'' \emph{Information and Computation}, vol. 138, no.~2, pp.
  170--193, 1997.

\bibitem{long1999complexity}
P.~M. Long, ``The complexity of learning according to two models of a drifting
  environment,'' \emph{Machine Learning}, vol.~37, no.~3, pp. 337--354, 1999.

\bibitem{freund1997learning}
Y.~Freund and Y.~Mansour, ``Learning under persistent drift,'' in
  \emph{European Conference on Computational Learning Theory}.\hskip 1em plus
  0.5em minus 0.4em\relax Springer, 1997, pp. 109--118.

\bibitem{littlestone1988learning}
N.~Littlestone, ``Learning quickly when irrelevant attributes abound: A new
  linear-threshold algorithm,'' \emph{Machine learning}, vol.~2, no.~4, pp.
  285--318, 1988.

\bibitem{Bengio2013Representation}
Y.~Bengio, A.~Courville, and P.~Vincent, ``Representation learning: A review
  and new perspectives,'' \emph{IEEE transactions on pattern analysis and
  machine intelligence}, pp. 1798--1828., 2013.

\bibitem{shen2018causally}
Z.~Shen, P.~Cui, K.~Kuang, B.~Li, and P.~Chen, ``Causally regularized learning
  with agnostic data selection bias.'' in \emph{ACM Multimedia}, 2018, pp.
  411--419.

\bibitem{Y1997Handwritten}
Y.~L. Cun, B.~Boser, J.~Denker, D.~Henderson, and L.~Jackel, ``Handwritten
  digit recognition with a back-propagation network,'' \emph{Advances in neural
  information processing systems}, vol.~2, no.~2, pp. 396--404, 1997.

\bibitem{krizhevsky2012imagenet}
A.~Krizhevsky, I.~Sutskever, and G.~E. Hinton, ``Imagenet classification with
  deep convolutional neural networks,'' in \emph{Advances in neural information
  processing systems}, 2012, pp. 1097--1105.

\bibitem{Ekin2018Autoaugment}
D.~M. V. V. Q. V.~L. Ekin D.~Cubuk, Barret~Zoph, ``Autoaugment: Learning
  augmentation policies from data,'' \emph{arXiv preprint arXiv: 1805.09501},
  2018.

\bibitem{takahashi2018ricap}
R.~Takahashi, T.~Matsubara, and K.~Uehara, ``Ricap: Random image cropping and
  patching data augmentation for deep cnns,'' in \emph{Asian Conference on
  Machine Learning}, 2018, pp. 786--798.

\bibitem{krizhevsky2009learning}
A.~Krizhevsky, G.~Hinton \emph{et~al.}, ``Learning multiple layers of features
  from tiny images,'' Citeseer, Tech. Rep., 2009.

\bibitem{he2016deep}
K.~He, X.~Zhang, S.~Ren, and J.~Sun, ``Deep residual learning for image
  recognition,'' in \emph{Proceedings of the IEEE conference on computer vision
  and pattern recognition}, 2016, pp. 770--778.

\bibitem{huang2017densely}
G.~Huang, Z.~Liu, L.~Van Der~Maaten, and K.~Q. Weinberger, ``Densely connected
  convolutional networks,'' in \emph{Proceedings of the IEEE conference on
  computer vision and pattern recognition}, 2017, pp. 4700--4708.

\bibitem{Zeiler2013Visualizing}
M.~D. Zeiler and R.~Fergus, ``Visualizing and understanding convolutional
  networks,'' 2013.

\bibitem{tishby2015deep}
N.~Tishby and N.~Zaslavsky, ``Deep learning and the information bottleneck
  principle,'' in \emph{2015 IEEE Information Theory Workshop (ITW)}.\hskip 1em
  plus 0.5em minus 0.4em\relax IEEE, 2015, pp. 1--5.

\bibitem{shwartz2017opening}
R.~Shwartz-Ziv and N.~Tishby, ``Opening the black box of deep neural networks
  via information,'' \emph{arXiv preprint arXiv:1703.00810}, 2017.

\bibitem{Mcallester1999PAC}
D.~A. Mcallester, ``Pac-bayesian model averaging,'' in \emph{Twelfth Conference
  on Computational Learning Theory}, 1999.

\bibitem{Vaart2013Weak}
A.~W. V.~D. Vaart and J.~A. Wellner, ``Weak convergence and empirical
  processes,'' \emph{Springer}, vol.~30, no.~4, pp. 355--373, 2013.

\end{thebibliography}

\appendix

\section*{A  Proof of Theorem \ref{thm2}}
\label{ap2}
\begin{thm}\label{thm2}
Suppose that hypothesis space $\mathcal{H}$ is $G_{Y}$-determining, $\forall \epsilon,\delta \in(0,1)$, for any hypothesis $h \in \mathcal{H}$, with probability of at least $1-\delta$, we have
\begin{equation}
\left | L_{S}(h)-L_{D}(h) \right |\leq \sqrt{\frac{2(TKln2+ln(1/\delta))}{n}}
\end{equation}
where $n$ is the size of the training set $S$, $K$ is the number of labels in the label space $\mathcal{Y}$, and $T$ is the number of different variable configurations of $\boldsymbol{G_Y}$.
\end{thm}
\begin{proof}
According to the ITID assumption, all the distributions which the data are sampled from have the same marginal distribution $D(\boldsymbol{G_Y},Y)$. At the same time, $G_Y$-determining means that the predictions will be identical with the same value of $\boldsymbol{G_Y}$, so we can use $\psi_{h}$ to represent the predictions from hypothesis $h$ for $\boldsymbol{G_Y}$. As a result, we can get the expected risk by using the distribution $D(\boldsymbol{G_Y},Y)$.
\begin{equation}
L_{D}(h)=1-\sum_{\boldsymbol{G_Y}}D(\boldsymbol{G_Y},\psi_{h}(\boldsymbol{G_Y}))
\end{equation}

Suppose that $\boldsymbol{G_Y}$ have $T$ different values: $\{g_i\}_{i=1..T}$, $Y$ have $K$ different labels: $\{l_j\}_{j=1..K}$, $\tau_i=\psi_{h}(g_i)$, then the expected risk can be written as:
\begin{equation}
L_{D}(h)=1-\sum_{i=1}^{T}D(g_i,\tau_i)
\end{equation}
The empirical risk can be written as:
\begin{equation}
L_{S}(h)=1-\sum_{i=1}^{T}\frac{N_{g_i,\tau_i}}{n}
\end{equation}
where $N_{g_i,\tau_i}$ represents the number of samples that have $\boldsymbol{G_Y}=g_i$ and $Y=\tau_i$, and $n$ is the size of training set.
We further have:
\begin{equation}\label{gap}
\begin{split}
\left | L_{S}(h)-L_{D}(h) \right |\\
&=\left | \sum_{i=1}^{T}\frac{N_{g_i,\tau_i}}{n}-\sum_{i=1}^{T}D(\boldsymbol{G_Y}=g_i,Y=\tau_{i}) \right |\\
&\leq \sum_{i=1}^{T}\left | \frac{N_{g_i,\tau_i}}{n}-D(g_i,\tau_{i}) \right |\\
&\leq \sum_{i=1}^{T}\sum_{j=1}^{K}\left | \frac{N_{g_i,l_j}}{n}-D(g_i,l_j) \right |\\
\end{split}
\end{equation}
Note that the choice of the hypothesis $h$ only affected $\tau_{i}$, and the above new upper bound is independent of $\tau_{i}$, so theoretically this conclusion can fit to all possible hypothesis.
Considering that the random vector $(N_{g_i,l_j})$ is multinomial distributed with parameters $n$ and $D(g_i,l_j)$, then according to the Breteganolle-Huber-Carol inequality
(Proposition A6.6 of \cite{Vaart2013Weak}), we have
\begin{equation}
\begin{split}
&Pr \left \{\sum_{i=1}^{T}\sum_{j=1}^{K}\left | \frac{N_{g_i,l_j}}{n}-D(g_i,l_j) \right |  \geq \lambda \right \} \\
&\leq 2^{TK}exp(\frac{-n\lambda^{2}}{2})
\end{split}
\end{equation}
Hence, the following holds with probability of at least $1-\delta$
\begin{equation}
\begin{split}
\sum_{i=1}^{T}\sum_{j=1}^{K}\left | \frac{N_{g_i,l_j}}{n}-D(g_i,l_j) \right | \\
\leq \sqrt{\frac{2(TKln2+ln(1/\delta))}{n}}
\end{split}
\end{equation}
Combing with Eqn.~\eqref{gap} ,we finally get
\begin{equation}
\begin{split}
\left | L_{S}(h)-L_{D}(h) \right |&\leq \sum_{i=1}^{T}\sum_{j=1}^{K}\left | \frac{N_{g_i,l_j}}{n}-D(g_i,l_j) \right |\\
&\leq \sqrt{\frac{2(TKln2+ln(1/\delta))}{n}}
\end{split}
\end{equation}

\end{proof}

\section*{B Proof of Theorem \ref{thm3}}
\label{ap3}
\begin{thm}\label{thm3}
Suppose hypothesis space $H$ is $\gamma$-$G_{U}$-dependence: if $n\geq m_{\mathcal{H}}(\frac{\epsilon}{2},\delta)$, with probability of at least $1-\delta$, $\forall h \in \mathcal{H}$ we have
\begin{equation}
L_{D}(h)\leq \min\limits_{h'\in \mathcal{H}} L_{D}(h')+\epsilon+\frac{\gamma}{log2}
\end{equation}
where $m_{\mathcal{H}}$ is derived from any uniform convergency bound under $G_Y$-determining constraint: if $n\geq m_{\mathcal{H}}(\epsilon,\delta)$, with probability of at least $1-\delta$, we have
\begin{equation}
\left | L_{S}(h)-L_{D}(h) \right |\leq \epsilon
\end{equation}
\end{thm}
\noindent Before proving theorem, we first derive the following lemma.

\begin{lemma}\label{lm1}
Suppose that $P(Y)$ is a probability distribution defined on the $K$ label space $\mathcal{Y}$, $\sum_{i=1}^{K}{P(Y=y_i)}=1$, and
\begin{equation}\label{ch}
H(Y)=\sum_{i=1}^{K}{-P(Y=y_i) logP(Y=y_i)}
\end{equation}
then we have:
\begin{equation}
\max\limits_{Y}P(Y)\geq 1-\frac{H(Y)}{2log2}
\end{equation}
\end{lemma}

\begin{proof}
Note that $0\leq P(Y)\leq1$, when $H(Y)\geq 2log2$, the conclusion is trivial. So we focus on the case when $0\leq H(Y)\leq 2log2$.
For simplicity, we denote $p_i=P(Y=y_i)$, and have
\begin{equation}\label{13}
H(Y)=\sum_{i=1}^{K}{-p_i logp_i}\leq 2log2
\end{equation}
Without loss of generality, let $p_1$ be the maximum of $\{p_i\}_{i=1..K}$, the problem transfers to find the lower bound of $p_1$.

In the following, we first obtain the minimum of $H(Y)$ when $p_1$ is determined, then use the minimum to prove the Eqn.~\eqref{13}.

(1) obtaining the minimum of $H(Y)$ when $p_1$ is given:
\begin{equation}
\label{raw}
\begin{split}
t(p_1)&=\min\limits_{p_2,p_3,...,p_K}{H(Y)}\\
&s.t. \\
&\sum_{i=1}^{K}{p_i}=1 \\
&0\leq p_i \leq 1-p_1, i=2,3,...,K-1
\end{split}
\end{equation}
In simpler terms, we use the constraint to substitute $p_K$, we have
\begin{equation}
\begin{split}
t(p_1)&=H(Y)=\sum_{i=1}^{K-1}{-p_i logp_i}-(1-\sum_{i=1}^{K-1}{p_i})log(1-\sum_{i=1}^{K-1}{p_i})\\
&s.t. \\
&0\leq p_i \leq 1-p_1, i=2,3,...,K-1
\end{split}
\end{equation}
As $-H(Y)$ is convex and derivable, the minimal of $H(Y)$ must be the boundaries. By examining all the boundary points, we can get one of the minimal:

\begin{equation}
\begin{split}
&p_i=0,i=2,3,...,K-1
\end{split}
\end{equation}

Note that such substitution ignores the constraint of $p_K$, so we should validate whether the minimal meets the raw constraints. As $p_K=1-p_1$ meets the constraints, it is also the minimal of Equation \ref{raw}, and
\begin{equation}
\begin{split}
t(p_1)=-p_1logp_1-(1-p_1)log(1-p_1)
\end{split}
\end{equation}

(2)
Note that $t(p_1)$ is concave and $0.5\leq p_1\leq 1$, we have
\begin{equation}
\begin{split}
t(p_1)\geq (1-p_1)\times2log2
\end{split}
\end{equation}
then
\begin{equation}
\begin{split}
&t(p_1)\leq H(Y)\\
&\Rightarrow 2log2(1-p_1) \leq H(Y)\\
&\Rightarrow p_1 \geq 1-\frac{H(Y)}{2log2}
\end{split}
\end{equation}
so
\begin{equation}
\begin{split}
\max\limits_{Y}P(Y)=p_1\geq 1-\frac{H(Y)}{2log2}
\end{split}
\end{equation}
\end{proof}

\noindent With Lemma 1, we can started to prove Theorem \ref{thm3}.
\begin{proof}
\emph{to Theorem 2}

Since the hypothesis space $\mathcal{H}$ is $\gamma$-$G_Y$-determining, we have
\begin{equation}
\sum_{\boldsymbol{G_Y}}{P(\boldsymbol{G_Y})\sum_{\hat{Y}}{-P(\hat{Y}|\boldsymbol{G_Y})logP(\hat{Y}|\boldsymbol{G_Y})}}\leq \gamma
\end{equation}
which denotes that
\begin{equation}
\Lambda(\boldsymbol{G_Y})=\sum_{\hat{Y}}{-P(\hat{Y}|\boldsymbol{G_Y})logP(\hat{Y}|\boldsymbol{G_Y})}
\end{equation}
then
\begin{equation}
\sum_{\boldsymbol{G_Y}}{P(\boldsymbol{G_Y})\Lambda(\boldsymbol{G_Y})}\leq \gamma
\end{equation}
and
\begin{equation}
\Theta(\boldsymbol{G_Y})=1-\max\limits_{\hat{Y}}{P(\hat{Y}|\boldsymbol{G_Y})}
\end{equation}
According to lemma \ref{lm1}, we have
\begin{equation}
\Theta(\boldsymbol{G_Y})\leq \frac{\Lambda(\boldsymbol{G_Y})}{2log2}
\end{equation}
so

\begin{equation}
P_A=\sum_{\boldsymbol{G_Y}}{P(\boldsymbol{G_Y})\Theta(\boldsymbol{G_Y})}\leq \sum_{\boldsymbol{G_Y}}{P(\boldsymbol{G_Y})\frac{\Lambda(\boldsymbol{G_Y})}{2log2}}\leq \frac{\gamma}{2log2}
\end{equation}

Note that $\arg\max\limits_{\hat{Y}}P(\hat{Y}|G_Y)$ is a $G_Y$-determining model, so the predictions are different from such model with probability $P_A$. Then we can get the following bound:

if $n\geq m_{\mathcal{H}}(\frac{\epsilon}{2},\delta)$, with probability of at least $1-\delta$, we have
\begin{equation}
\left | L_{S}(h)-L_{D}(h) \right |\leq \frac{\epsilon}{2}+\frac{\gamma}{2log2}
\end{equation}
Then, as the learning algorithm is based on empirical risk minimization, we have:
\begin{equation}
\begin{split}
L_{D}(h) &\leq L_{S}(h)+\frac{\epsilon}{2}+\frac{\gamma}{2log2} \leq L_{S}(h')+\frac{\epsilon}{2}+\frac{\gamma}{2log2} \\
&\leq L_{D}(h')+2(\frac{\epsilon}{2}+\frac{\gamma}{2log2})
\end{split}
\end{equation}
Therefore we derive the final conclusion:

if $n\geq m_{\mathcal{H}}(\frac{\epsilon}{2},\delta)$, with probability of at least $1-\delta$, we have
\begin{equation}
L_{D}(h)\leq \min\limits_{h'\in \mathcal{H}} L_{D}(h')+\epsilon+\frac{\gamma}{log2}
\end{equation}
\end{proof}

\section*{C Proof of Theorem \ref{thm4}}
\label{ap4}
\begin{thm}\label{thm4}
Suppose that $\hat{Y}$ is the predictions of hypothesis $h$, and we have $N$ task-uncorrelated variables: $G_U=\{G_U^{(i)}\}_{i=1..N}$. If
\begin{equation}
\gamma=\sum_{i=1}^{N}{I(\hat{Y};G_U^{(i)}|\boldsymbol{G_Y})}
\end{equation}
then $h$ is $\gamma$-$G_{U}$-dependence. Where $I$ represents the mutual information.
\end{thm}

\begin{proof}
\begin{equation}
\begin{split}
\gamma&=\sum_{i=1}^{N}{I(\hat{Y};G_U^{(i)}|\boldsymbol{G_Y})}\geq I(\hat{Y};\boldsymbol{G_U}|\boldsymbol{G_Y})\\
&=H(\hat{Y}|\boldsymbol{G_Y})-H(\hat{Y}|\boldsymbol{G_Y},\boldsymbol{G_U})
\end{split}
\end{equation}
Note that $(\boldsymbol{G_Y},\boldsymbol{G_U})$ represents all the generative variables, so we have $H(\hat{Y}|\boldsymbol{G_Y},\boldsymbol{G_U})=0$. Then
\begin{equation}
\begin{split}
H(\hat{Y}|\boldsymbol{G_Y})\leq \gamma
\end{split}
\end{equation}
Therefore, $h$ is $\gamma$-$G_Y$-determining.
\end{proof}

\section*{D Proof of Theorem \ref{thm5}}
\label{ap5}
\begin{thm}\label{thm5}
Suppose that the hypothesis class $\mathcal{H}$ is strict-$G_l$-determining, and the optimal hypothesis $h'\in\mathcal{H}$ is obtained by cross-entropy loss with empirical risk minimization. The output values $(Q_0^{'},Q_1^{'},...,Q_K^{'})$ of the optimal hypothesis $h'(x)$ under such settings will be
\begin{equation}
Q_y^{'}(x)=\widetilde{P}_S(Y=y|\boldsymbol{G_l}=G_l(x))
\end{equation}
where $\widetilde{P}_S$ represents the empirical distribution of training set $S$.
\end{thm}
\begin{proof}
The cross entropy loss on training set $S$ is:
\begin{equation}
CELoss=\frac{1}{n}\sum_{(x,y)\in S}{-logQ_y(x)}
\end{equation}

As $h$ is strict-$G_l$-determining, it will provide a constant output value for the same $G_l$ and the same label. We use $\psi_{h}:\mathcal{G}_l \times \mathcal{Y}\mapsto [0,1]$ to denote the mapping. Note that $Q_y(x)$ is the $y^{th}$ output value of $h$ when the input is $x$, we have:
\begin{equation}
\begin{split}
Q_y(x)=\psi_{h}(G_l(x),y)
\end{split}
\end{equation}
We have another alternative representation of cross entropy loss:
\begin{equation}
\begin{split}
CELoss&=\sum_{\boldsymbol{G_l}}{\sum_Y{-\widetilde{P}_S(\boldsymbol{G_l},Y)log\psi_h(\boldsymbol{G_l},Y)}}\\
&=\sum_{\boldsymbol{G_l}}{\widetilde{P}_S(\boldsymbol{G_l})\sum_Y{-\widetilde{P}_S(Y|\boldsymbol{G_l})log\psi_h(\boldsymbol{G_l},Y)}}
\end{split}
\end{equation}
with the following constraints:
\begin{equation}
\begin{split}
&\sum_Y{\widetilde{P}_S(Y|\boldsymbol{G_l})}=1 \\
&\sum_Y{\psi_h(\boldsymbol{G_l},Y)}=1
\end{split}
\end{equation}

As minimizing cross entropy loss is a convex optimization problem, we can obtain minima when gradients are all zero:
\begin{equation}\label{solution}
\begin{split}
\psi_h(\boldsymbol{G_l},Y)=\widetilde{P}_S(Y|\boldsymbol{G_l})
\end{split}
\end{equation}
So the optimal hypothesis will be:
\begin{equation}
Q_y^{'}(x)=\psi_{h'}(G_l(x),y)=\widetilde{P}_S(Y=y|\boldsymbol{G_l}=G_l(x))
\end{equation}
\end{proof}

\section*{E Proof of Theorem \ref{thm6}}
\begin{thm}\label{thm6}
Suppose that the hypothesis class $\mathcal{H}$ is strict-$G_l$-determining, and obtain the optimal hypothesis $h'\in\mathcal{H}$ is obtained by cross-entropy loss with empirical risk minimization. Given a generative variable subset $G_t\subset G_l$, we have the sufficient condition for that the optimal $h'$ is strict-$G_t$-invariance on the training set $S$:

 $G_t$ is statistically independent with $G_t^c$ and $Y$ on the training set $S$.

\noindent where $G_t^c$ is the complementary set of $G_t$ in $G_l$.
\end{thm}
\begin{proof}
According to Theorem \ref{thm5}, we have
\begin{equation}
Q_y^{'}(x)=\widetilde{P}_S(Y=y|\boldsymbol{G_t}=G_t(x))
\end{equation}

Note that $\widetilde{P}_S(\boldsymbol{G},Y)=\widetilde{P}_S(Y|\boldsymbol{G})\widetilde{P}_S(\boldsymbol{G})$, and $G_t$ is statistically independent of $G_t^c$ and $Y$. Therefore, we have
\begin{equation}\label{1}
\begin{split}
\widetilde{P}_S(Y|\boldsymbol{G})&=\frac{\widetilde{P}_S(\boldsymbol{G},Y)}{\widetilde{P}_S(\boldsymbol{G})}=\frac{\widetilde{P}_S(\boldsymbol{G_t},\boldsymbol{G_t^c},Y)}
{\widetilde{P}_S(\boldsymbol{G_t},\boldsymbol{G_t^c})} \\
&=\frac{\widetilde{P}_S(\boldsymbol{G_t})\widetilde{P}_S(\boldsymbol{G_t^c},Y)}{\widetilde{P}_S(\boldsymbol{G_t})\widetilde{P}_S(\boldsymbol{G_t^c})}
=\frac{\widetilde{P}_S(\boldsymbol{G_t^c},Y)}{\widetilde{P}_S(\boldsymbol{G_t^c})} \\
&=\widetilde{P}_S(Y|\boldsymbol{G_t^c})
\end{split}
\end{equation}
so we have:
\begin{equation}
Q_y^{'}(x)=\widetilde{P}_S(Y=y|\boldsymbol{G}=G(x))=\widetilde{P}_S(y|G_t^c(x))
\end{equation}
the hypothesis is strict-$G_t$-invariant on the training set $S$.
\end{proof}
\section*{F Proof of Corollary 2}
\label{ac2}
\begin{cor}\label{cor2}
Given the optimal hypothesis $h'$ satisfying Theorem \ref{thm5}, we have the training error
\begin{equation}\label{eq-cor2}
L_{S}(h')=1-\mathbb{E}_{x\in S}\max_{y\in \mathcal{Y}}{Q'_y(x)}
\end{equation}
where $\mathbb{E}_{x\in S}$ is the expectation over $S$.
\end{cor}
\begin{proof}
As $h'$ satisfies Theorem \ref{thm5} and the prediction is the label with maximum output value, if the input is $x$, the hypothesis will derive the prediction:
\begin{equation}
\begin{split}
\max_{y\in \mathcal{Y}}Q'_y(x)=\max_{y\in \mathcal{Y}}\psi_{h'}(G_l(x),y)
\end{split}
\end{equation}
so the training error can be written as:
\begin{equation}
\begin{split}
L_{S}(h')&=1-\frac{1}{n}\sum_{i=1}^{n}{\mathbb{I}(\arg\max_{y\in \mathcal{Y}}Q'_y(x_i),y_i)} \\
&=1-\sum_{i=1}^{T}\widetilde{P}_S(\boldsymbol{G_l}=g_i,Y=\arg\max_{y\in \mathcal{Y}}\psi_{h'}(g_i,y)) \\
&=1-\sum_{i=1}^{T}\widetilde{P}_S(Y=\arg\max_{y\in \mathcal{Y}}\psi_{h'}(g_i,y)|\boldsymbol{G_l}=g_i)\widetilde{P}_S(\boldsymbol{G_l}=g_i) \\
&=1-\mathbb{E}_{g\in S}\widetilde{P}_S(Y=\arg\max_{y\in \mathcal{Y}}\psi_{h'}(g,y)|g)
\end{split}
\end{equation}
According to Theorem \ref{thm5}:
\begin{equation}
\begin{split}
L_{S}(h')&=1-\mathbb{E}_{g\in S}\widetilde{P}_S(Y=\arg\max_{y\in \mathcal{Y}}\psi_{h'}(g,y)|g)\\
&=1-\mathbb{E}_{g\in S}\psi_{h'}(g,\arg\max_{y\in \mathcal{Y}}\psi_{h'}(g,y))\\
&=1-\mathbb{E}_{g\in S}\max_{y\in \mathcal{Y}}\psi_{h'}(g,y)\\
&=1-\mathbb{E}_{x\in S}\max_{y\in \mathcal{Y}}{Q_y(x)}
\end{split}
\end{equation}

\end{proof}

\end{document}